\documentclass[final,3p,times,twocolumn]{article}

\usepackage{graphics} 
\usepackage{amsmath} 
\usepackage{amssymb}  
\usepackage[table]{xcolor}
\usepackage[ruled,linesnumbered, noend]{algorithm2e}

\usepackage{soul}
\usepackage{comment}
\usepackage{multirow}
\usepackage{graphicx}
\usepackage{wrapfig}
\usepackage{balance}
\usepackage{subfigure}
\usepackage{capt-of}
\usepackage{natbib}

\usepackage{amsthm}

\newtheorem{mydef}{Definition}
\newtheorem{thm}{Theorem}

\newcounter{example}[section]

\newcommand{\vect}[1]{\mathbf{#1}} 
\newcommand{\ve}[1]{\mathbf{#1}} 
\newcommand{\ssl}[1]{\mathcal{#1}} 

\newcommand{\x}[1]{\mathbf{x}_{#1}}

\begin{document}

\title{Simple Kinesthetic Haptics for Object Recognition}

\author{Avishai~Sintov and Inbar Ben-David\footnote{School of Mechanical Engineering, Tel-Aviv University, Israel.}}

\maketitle

\begin{abstract}
Object recognition is an essential capability when performing various tasks. Humans naturally use either or both visual and tactile perception to extract object class and properties. Typical approaches for robots, however, require complex visual systems or multiple high-density tactile sensors which can be highly expensive. In addition, they usually require actual collection of a large dataset from real objects through direct interaction. In this paper, we propose a kinesthetic-based object recognition method that can be performed with any multi-fingered robotic hand in which the kinematics is known. The method does not require tactile sensors and is based on observing grasps of the objects. We utilize a unique and frame invariant parameterization of grasps to learn instances of object shapes. To train a classifier, training data is generated rapidly and solely in a computational process without interaction with real objects. We then propose and compare between two iterative algorithms that can integrate any trained classifier. The classifiers and algorithms are independent of any particular robot hand and, therefore, can be exerted on various ones. We show in experiments, that with few grasps, the algorithms acquire accurate classification. Furthermore, we show that the object recognition approach is scalable to objects of various sizes. Similarly, a global classifier is trained to identify general geometries (e.g., an ellipsoid or a box) rather than particular ones and demonstrated on a large set of objects. Full scale experiments and analysis are provided to show the performance of the method.
\end{abstract}



\section{Introduction}
\label{sec:introduction}

When humans reach for an object, such as a mobile phone or the computer mouse, they can identify the object even while their gaze is directed elsewhere \citep{Lederman1993}. Some finger contact with the object is usually sufficient to acquire essential information to instantly identify the object \citep{Klatzky1985, Klatzky1995}. The notion of \textit{Exploratory procedures} (EP), defined by \cite{Lederman1987}, hypothesized that haptic exploration typically involves a stereotypical set of hand gestures, each aimed to identify a specific property of the object. These include lateral motion, pressure and some arbitrary contact for texture, stiffness and temperature sensing. More relevant topics are finger enclosure and contour following for shape or volume recognition \citep{Driess2017}.

\begin{figure}
\centering
\includegraphics[width=\linewidth]{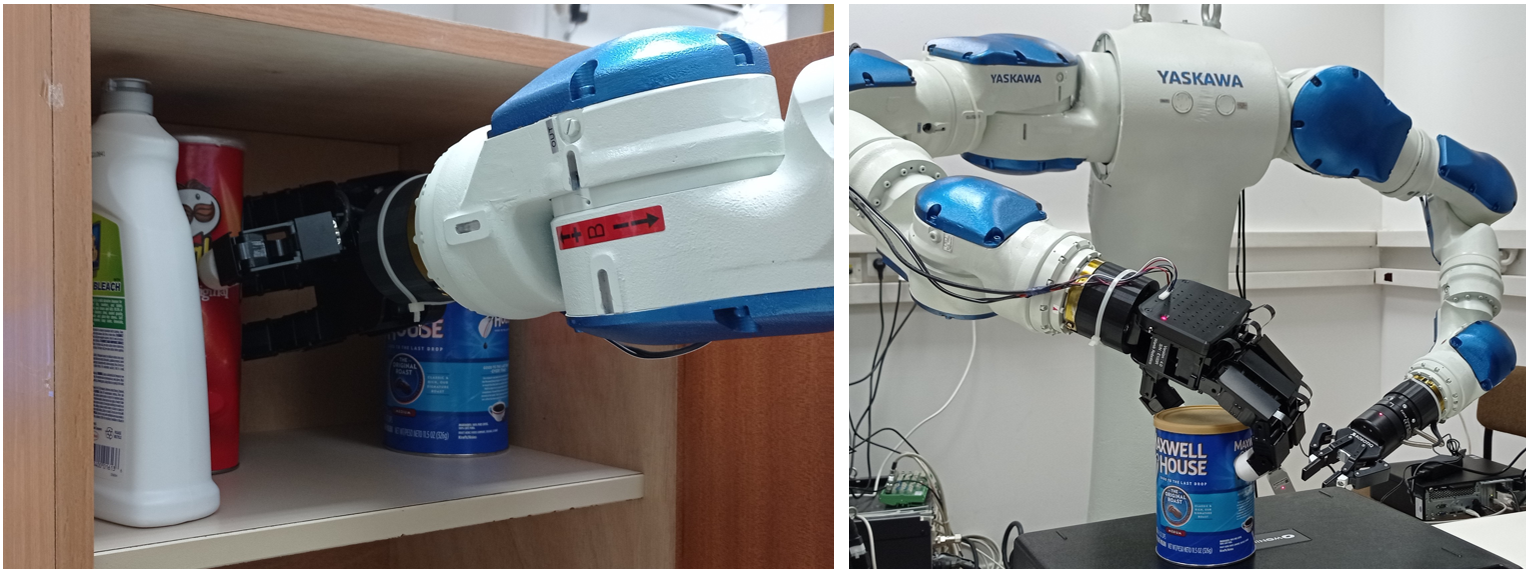}
\caption{Robotic haptic glances to identify (left) an occluded object in a cabinet and (right) on a table.}
\label{fig:allegro_grasp}
\end{figure}

Inspired by these studies, robotic hands have been developed to include haptic capabilities, enabling the identification of objects and their properties \citep{Dario1992}. Contrary to visual sensing, haptic sensors are capable of directly perceiving physical properties of an object such as shape,  compliance, texture and temperature. With haptic sensing technology, robots can extract information of an object even when the object is fully or partially occluded to a visual system. This can happen when, for instance, searching for an object at the back of the closet or in a box full of items.

Information from haptic sensors is acquired through direct interaction with objects usually involving both tactile sensing and internal sensing of joint actuators known as Kinesthetic haptics \citep{Carter2005}. Traditionally, tactile refers to information received from touch and contact sensing, while kinesthetic refers to information sensed through movement, force or position of joints and actuators. Tactile sensing with or without visual perception is the leading method for haptic-based object recognition \citep{Rouhafzay2020}. State of the art in tactile sensing for identifying and classifying objects usually involves several high-density tactile sensors on the robotic hand which can be highly expensive and may complicate the hardware \citep{Drimus2011}. Furthermore, training classifiers to process tactile information requires actual data collection from a set of objects \citep{Spiers2016,Lin2019,Jasper2020}. Hence, even if tactile sensors exist, actual collection is required and may not be transferable to another hand. Simulating viable tactile signals, on the other hand, is a challenging task and cannot be used straight-forward to train a classifier \citep{Narang2020}. In addition, most of the existing work can only deal with simple shaped objects \citep{Liu2017}. Previous work on object recognition without tactile sensors suggested to map the surface of the object by collecting a point cloud of contact points using a robotic hand \citep{Allen1989,Jin2013}. The data is then used to fit a parametric surface model. The method, however, requires a large dataset of actual measurements and simplifies the model to some quadratic form. 

In our work, we propose a kinesthetic method to classify objects with neither tactile sensors nor visual perception, and without data collection on actual objects. A grasp embeds an instance of the surface of held objects. Different objects have different patterns of grasps based on their shapes. Such grasp, however, is usually described by the location of the contact points and depends on the used reference frame. Hence, a grasp described by the contact points is not a unique representation and may be challenging when used to train a classifier. On the other hand, a grasp parameterization technique proposed by \cite{Sintov2014124,Sintov3D2016} is unique and independent of any coordinate frame. The proposed representation treats a spatial grasp as a polyhedron and provides an injective parameterization. One may add more descriptive information of the object by also including the normals at the contact points relative to the polyhedron. 

\begin{figure*}[!ht]
\centering
\includegraphics[width=\linewidth]{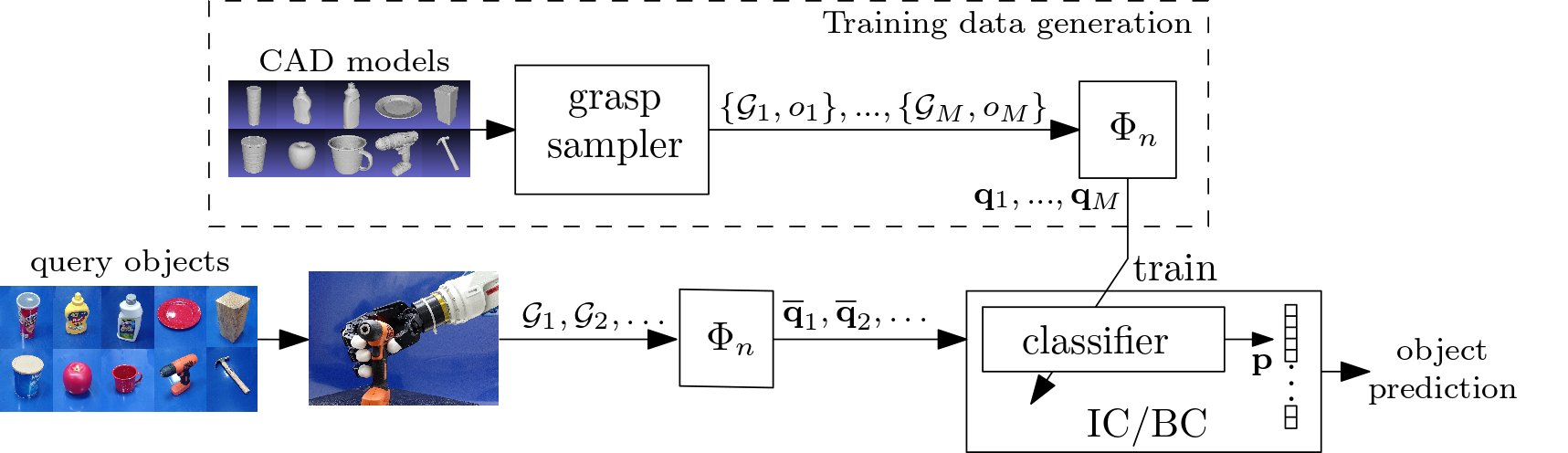}
\caption{Illustration of the (top) training and (bottom) evaluation processes. CAD models of the objects are used to generate labeled data for training a classifier. The classifier, trained solely over CAD models and independent of any particular robot hand, can now be used to classify real object. Recognition of real objects is done by random sampling grasps using a robotic hand in an iterative classification process - Iterative Classification (IC) or Bayesian Classification (BC).}
\label{fig:scheme}
\end{figure*}

In this paper, we show that with the proposed grasp representation
, a classifier can be trained to identify complex objects within several grasps. The classifier is trained solely on object information based on its geometry without the need of a physics-engine simulator. The classifier is independent of any particular robot hand and, therefore, can be applied to various ones. A grasp, in this context, is an \textit{haptic glance} \citep{Dragiev2013} aimed to gain more information and improve the classes probability distribution. Such grasp is demonstrated by the Allegro hand in Figure \ref{fig:allegro_grasp}. 
Generating training data is solely done computationally by sampling contact points from the surface of Computer-Aided Design (CAD) models of the objects independent of the type of the used hand. Hence, adding an object is fast (up to several minutes). If a CAD model is not available, the object can be scanned or actual data collection with a robotic hand is possible. Actual grasp information for object recognition queries is acquired through kinesthetic sensing, i.e., measuring joint angles and torques of the fingers while grasping the query object. 

We assume that the general location of the object is a priori known. The location of the object can be acquired through low-level object localization without complex Convolutional Neural Network classification \citep{Du2020}. For instance, 2D salient object detection can be applied through visual perception with a simple camera either on the hand or around the robot. We leave the object localization problem out of the scope of this paper.

When classifying an object, one grasp generally does not include enough surface information. Therefore, we propose and investigate the use of iterative methods to improve predictions, and provide a comparative analysis between methods. An illustration of the process is shown in Figure \ref{fig:scheme}. We propose to use either a scoring-based iterative method or Bayesian update to approximate the certainty about an object based on a few grasp samples. We also investigate the ability to classify an object with partial grasp information solely with finger contact location while not including the normals at the contact points. Furthermore, we explore the ability of the technique to generalize to families of objects (e.g., ellipsoids, cylinders) and not only specific objects with available CAD models. Hence, we provide a global classifier for any multi-finger robotic hand to recognize general types of objects. We also exhibit a scale invariant property of the method when simply normalizing the grasp polyhedron.


\section{Related work}
\label{sec:related_work}

Humans and animals use haptic sensory receptors and EP to perceive and interact with the environment \citep{Nelinger2016}. A study has identified neurons in the post-central gyrus of monkeys that are activated selectively to grasped objects while not activated when touching their own body \citep{Iwamura1995}. These neurons are tuned to specific features that discriminate common and familiar objects from the animal's own body. 
Inspired by these studies, robotic systems have been developed to
grant a robotic hand with such haptic capabilities and to identify objects. These efforts will now be discussed.


\subsection{Object Learning through tactile perception}

Tactile perception in robotics is commonly used to learn a priori unknown object in uncertain environments to grasp and manipulate it. The learning is generally used to either learn features of an object or for object recognition. The latter is discussed in the next section. Feature learning may include compliance \citep{Su2012}, texture \citep{Fishel2012, Wenzhen2017}, temperature variations \citep{Dario1992} and surface modeling. A common practice to model the surface of an unknown object is by iteratively selecting discrete locations, known as \textit{haptic glances}, where a robot touches \citep{Dragiev2013,Yi2016}. Tactile glances are used to improve an underlying surface model of the object using Gaussian process regression while incorporating model uncertainty. Similarly, tactile glances were used to augment noisy visual perception of an object \citep{Bjorkman2013}. In the tactile glances process, \cite{Jamali2016} proposed to incorporate a Gaussian process classifier to identify the boundary of the object and to filter-out glances that do not belong to it. More recently active exploration was proposed by \cite{Driess2017} where a robotic arm slides a probe over the surface of the object to acquire surface information. By exploring paths over the object, a continuous approximation of its surface is acquired. 

\subsection{Object Recognition}

Object recognition has been widely researched in the past few decades and can be divided into methods of visual \citep{Lowe1999} and tactile \citep{Okamura1997} perceptions. Both aim to identify an object in the vicinity of a robot. For visual perception, deep neural networks, or convolutional neural networks (CNN) in particular, have been making significant achievements in detection and classification of objects seen in 2D images \citep{Deng2009, Lin2014}. In contrary, classification based on 3D (depth) images is usually much more complicated due to higher dimensional representation of the data and the required computational effort \citep{Griffiths2019}. PointNet is a recent advancement where a unified neural network directly takes raw 3D point clouds for object classification and segmentation \citep{Charles2017}. In another work, a set of 2D images taken from a 3D CAD model was rendered to generate a multi-view CNN that classifies objects from different viewing angles \citep{Su2015}. While these methods achieve high success rates, they require a substantial computational power and a direct line of sight with an object limited to a viewing area that obscures its back-facing region. The object can also be fully or partly occluded. 

Contrary to vision, tactile-based recognition requires direct contact with an object. Supervised learning techniques are the preferred approach to map tactile feature data from an object to its class. Hence, the common research question lies in the type of data to measure. \cite{Drimus2014} developed flexible tactile sensors that provide pressure images during object squeezing. Tactile information corresponding to the pressure distribution during the squeezing was used with a nearest neighbor classifier to classify the object in hand. \cite{Luo2018} used patterns of tactile images along with kinesthetic information while utilizing $k$-means clustering to increase classification success rate. Similar tactile technology described the local shape of objects by the covariance matrix of the pressure image and used in a Naive Bayes classifier \citep{Liu2012}. Another work by \cite{Rouhafzay2020} combined visual perception to identify salient points over the surface of the objects in order to improve object recognition from a haptic image glance. 

\cite{Spiers2016} proposed a method to perform a single grasp of an object using an underactuated hand with embedded pressure sensors on each finger. A hybrid approach uses measured pressure data and a hand kinematic estimator to train a random forest classifier and a parametric method. The approach enables to determine the class of an object and its physical properties including stiffness and size. All of the above tactile methods require somewhat tedious actual collection of data for all objects prior to training a classifier. Furthermore, these methods make use of complex, high-density sensor systems on the robotic hand which can be highly expensive and complicate the hardware. Our proposed approach, however, is distinguished from previous work by not requiring any tactile sensors and by the ability to train a classifier solely with data generated computationally on CAD models of the objects.

The work by \cite{Allen1989} has some similarities to ours in which a large dataset of contact points is used to fit a quadratic surface model and further classify the shape. Further work included also normal information in the point cloud of contacts \citep{Jin2013}. Both methods simplify the query objects to a set of shape primitives such as spheres and cylinders. We, however, use the full representation of an object, do not reconstruct its surface and use a significantly lower amount of data to identify it.


\subsection{Grasp \& Object Parameterization}

Object shape parameterization has been widely researched for various applications \citep{Sarkar2017}. Much work has been done in the area of 3D shape similarity comparison. The work described by \cite{Muser2012}, \cite{Novotni2001} and \cite{Osada2002} applied algorithms that were, to that point, used for Internet and local database search, face recognition, image processing, or parts identification. The work of \cite{Ohbuchi2002} on shape similarity search uses a generalized feature vector of a 3D polygonal mesh constructed from the moment of inertia, average distance of the surface from the model's axis, and its variance. However, such methods deal with mean parameterization of the geometry (such as volume, shape distribution, moment of inertia) of the objects and cannot be applied for grasping. 

Inspired by the above, \cite{Li2005} combined the notion of object parameterization with grasping. The work is based on shape matching for finding the best grasp for a set of objects. The best grasp is found by matching hand poses from a database of objects and human grasp postures. This is done by using a predefined parameterization of the object surface and the hand poses. An entire representation of any $n$-finger grasps was later presented that uniquely parameterized polygons formed by planar grasps   \citep{SintovCASE2012,Sintov2014124} and polyhedrons by spatial ones \citep{Sintov3D2016}. Such parameterization was used to compare grasps of different objects to find a common grasp and the matching design of a simple non-dexterous gripper to grasp them all. The method was also applied for grasping sheet-metal parts \citep{RAM2015}. This parameterization is used in this work and presented next.


\section{Grasp representation}
\label{sec:grasp_rep}

In this section, we discuss a grasp representation based on previous work proposed by \cite{Sintov2014124,Sintov3D2016}. An $n$-finger grasp of object $\ssl{O}$ can be defined by a set of $n$ contact points, 
\begin{equation}
\ssl{P}=\left\{\x{i}:\x{i}\in\mathbb{R}^D~\text{for}~i=1,...,n\right\}
\label{eq:sslP}
\end{equation}
on the object's surface, and the normal to the object's surface at each point
\begin{equation}
\ssl{N}=\left\{\ve{n}_i:\ve{n}_i\in\mathbb{R}^D~\text{for}~i=1,...,n\right\},
\label{eq:sslN}
\end{equation}
where $\|\ve{n}_i\|=1$ and, $D=2$ or $D=3$ for planar and spatial objects, respectively. We require a map $\Phi_n$ that takes any grasp $\ssl{G}=\{\ssl{P},\ssl{N}\}$ defined in some reference frame and maps it into a $w$-dimensional injective vector $\vect{q}\in\ssl{Q}$ representing the grasp, where $\ssl{Q}\subset\mathbb{R}^w$. That is, we require a map
\begin{equation}
\label{eq:mapping}
\Phi_n:\ssl{G}\to \ssl{Q},
\end{equation}
such that $\Phi_n^{-1}(\ve{q}_j)\neq \Phi_n^{-1}(\ve{q}_k)$ for any $j\neq k$. We note that the grasp representation to be described is valid for $n\geq3$ while we will further exclude an $n=2$ finger grasp case.

While we do not present results for the planar case, for intuition and generality, we first describe a planar representation where $D=2$. An $n$-finger planar grasp can be presented by an $n$-vertex polygon, where each vertex is positioned at a contact point as seen in Figure \ref{fig:planar_grasp}. Such polygon can be generated by constructing concave or convex polygons from a set of points in a plane according to polygonization algorithms \citep{Deneen1988}. As shown by \cite{Sintov2014124}, an $n$-vertex polygon requires $2n-3$ constraints to define its shape and size. Explicitly, $n-1$ internal angles of the polygon $\{\gamma_1,...,\gamma_{n-1}\}$ along with $n-2$ edge lengths $\{d_1,\ldots,d_{n-2}\}$ sufficiently define its shape and size summing up to $2n-3$ parameters. Further, to describe each contact, we define the normal at the contact relative to the grasp polygon, i.e., the normal direction can be defined by its angle relative to an adjacent edge of the polygon yielding $n$ additional parameters $\{\theta_1,\ldots,\theta_n\}$. The resulting $w$-dimensional grasp vector is 
\begin{equation}
    \ve{q}=(\gamma_1,...,\gamma_{n-1},d_1,\ldots,d_{n-2},\theta_1,\ldots,\theta_n)
\end{equation}
where $w=3n-3$. For example, a 3-finger grasp will create a triangle parameterized with two angles and an edge length. Defining the normals will add three more angles yielding a 6-dimensional parameterization vector given by
\begin{equation}
    \ve{q}=(\gamma_1,\gamma_{2},d_1,\theta_1,\theta_2,\theta_3).
\end{equation}
Similarly, the 4-finger grasp seen in Figure \ref{fig:planar_grasp} forms a 9-dimensional vector.
Such parameterization is not yet injective since several parameter choices can be made for the same polygon. Therefore and as a convention, parameter $d_1$ is always taken to be the longest edge while $d_2,...,d_{n-2}$ are taken sequentially in the counter-clockwise direction.
\begin{figure}
\centering
\includegraphics[width=0.45\textwidth]{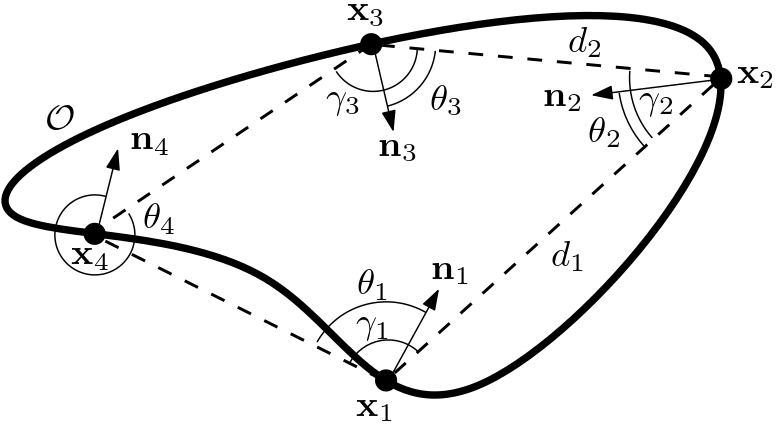}
\caption{Example of a 4-vertex polygon describing a 4-finger ($n=4$) grasp of object $\ssl{O}$ with a vector of nine parameters.}
\label{fig:planar_grasp}
\end{figure}

In the spatial case where $D=3$, the $n$ contact points $\ssl{P}$ form a polyhedron $\Theta(\ssl{P})$ with $n$ vertices as seen in Figure \ref{fig:spatial_grasp}. A triangulation of a set $\ssl{P}$ is the creation of simplices, where their vertices are points of $\ssl{P}$ such that the union of them equals $\Theta(\ssl{P})$ \citep{Avis1987}. Triangulation of $\ssl{P}$ can be done with algorithms such as the Quickhull \citep{Bradford-Barber1996} or the space sweep technique \citep{Preparata1985}. The algorithm will output a set of triangles forming $\Theta(\ssl{P})$. Similar to the planar case, we divide the parameterization vector into constraints for the shape and size, and the ones that define the contact normals. An $n$-vertex polyhedron requires $3n-6$ constraints to define its shape and size \citep{Sintov3D2016}. Here also, we may paramterize some of the triangles forming the polyhedron and their respected posture. However, there are many possible representations for the polyhedron depending on the order in which we parameterize its triangles. Hence, we parameterize the polyhedron in a specific order, ensuring an injective grasp representation. For brevity, we solely provide a brief overview of the algorithm while a full description is presented in previous work \citep{Sintov3D2016}. The algorithm first parameterizes the triangle (similar to the planar case - two angles $\gamma_1,\gamma_2$ and the longest edge length $d$) with the largest area, denoted as $t_1$, followed by the adjacent one $t_2$ through the longest edge. The parameterization moves on through a chain of $n-2$ adjacent triangles $t_j$, each with 3 parameters: two angles $\gamma_1^j,\gamma_2^j$ and the angle $\vartheta_j$ between triangle facets $t_j$ and $t_{j-1}$. Moreover, the algorithm describes the normals at the contact points relative to the polyhedron. Normal $\ve{n}_k$ needs two constraints in order to be defined and are given by two angles $\phi_1^k,\phi_2^k$ between the normal and triangles $t_1$ and $t_2$. For the 3-finger case, the two angles are defined relative to $t_1$ and the adjacent edge. Therefore, the resulting vector is of dimension $w=5n-6$ and is given by
\begin{multline}
    \ve{q}=(\gamma_1^1,\gamma_2^1,d,\gamma_1^2,\gamma_2^2,\vartheta_2,\ldots\\\gamma^{n-2}_1,\gamma^{n-2}_2,\vartheta_{n-2},\phi^1_1,\phi^1_2,\ldots\phi^n_1,\phi^n_2).
\end{multline}
For example, the 4-finger grasp in Figure \ref{fig:spatial_grasp} requires a 14-dimension vector: six for the geometry and eight for the normals:
\begin{equation}
    \ve{q}=(\gamma_1^1,\gamma_2^1,d,\gamma_1^2,\gamma_2^2,\vartheta_2,\phi^1_1,\phi^1_2,\ldots\phi^4_1,\phi^4_2).
\end{equation}
\begin{figure}
\centering
\includegraphics[width=0.35\textwidth]{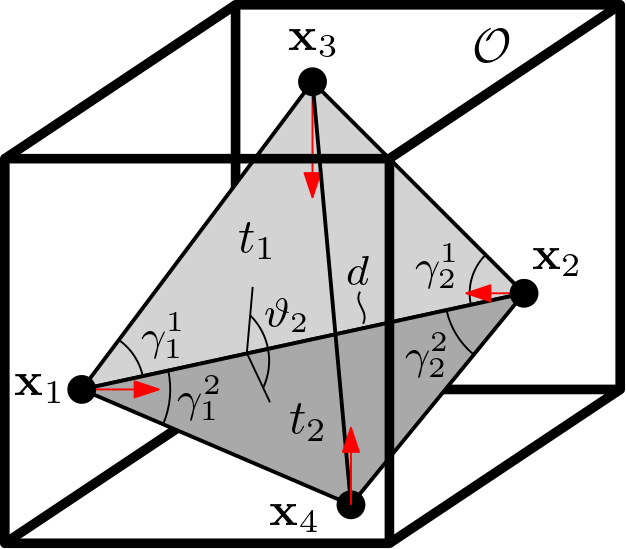}
\caption{Example of a 4-finger grasp of a box. A 4-vertex polyhedron formed by four contact points is defined by six parameters. The contact normals are marked in red while their eight parameterization angles are omitted for simplicity. }
\label{fig:spatial_grasp}
\end{figure}

Map $\Phi_n$ presents a parameterization of a grasp described by contact points on the object and the normals at the contacts. In both planar and spatial cases, the map provides a unique vector form for a grasp. It should be noted that this grasp representation is unique only when normals are included. We now argue that such mapping is also invariant to any coordinate frame.
\begin{thm}
\label{thm:frame_invariance}
Let object $\ssl{O}$ be described relative to coordinate frame $\ssl{C}_i$, and let coordinate frame $\ssl{C}_j$ be defined such that $\ve{c}_i=R\ve{c}_j+\ve{d}$ for any $\ve{c}_i\in\ssl{C}_i$ and $\ve{c}_j\in\ssl{C}_j$, and for some $R\in SO(3)$ and $\ve{d}\in\mathbb{R}^3$. Given a grasp $\ssl{G}_k=\{\ssl{P}_k,\ssl{N}_k\}$ of the object described in $\ssl{C}_i$, it must be that
\begin{equation}
    \Phi_n(\ssl{G}_k)=\Phi_n(\ssl{G}_k^{'})
\end{equation}
with $\ssl{G}_k^{'}=\{R\ssl{P}_k+\ve{d},R\ssl{N}_k\}$.
\end{thm}
\begin{proof}
See Appendix B.
\end{proof}
\noindent Theorem \ref{thm:frame_invariance} states that the above grasp representation is independent of any coordinate frame. Contact locations along with their normals are described relative to each other by a series of lengths and angles. The grasp representation is also not dependent on any robotic hand but only on contacts.

Measuring finger joint torques is not always possible or accurate enough and, therefore, the normals $\ssl{N}$ at the contact points are not always available. In such case, only the contacts are available and the grasp is partially defined by $\ssl{G}=\{\ssl{P}\}$. The corresponding map $\Phi_n$ outputs a vector $\ve{q}$ of size $3n-6$ only parameterizing the polyhedron formed by the contacts as previously described. In the above grasp parameterization, we have excluded 2-finger grasps ($n=2$) where the normals at the contact points have no unique parameterization. Hence, 2-finger grasps can only be parameterized with no normals by only considering the distance between the contacts (one dimensional representation). We next use the above grasp parameterization for object recognition.

\section{Object Recognition}
\label{sec:obj_classification}

In this section, we use the grasp representation described in Section \ref{sec:grasp_rep} to identify a grasped object. Given a set of $m$ objects $\{\ssl{O}_1,\ldots\ssl{O}_m\}$, we require to identify an object from the set without any use of visual feedback or tactile sensors. The query object will be identified solely by placing finger tips of a robotic hand on the surface of the object or grasping it. 
We assume that the robotic hand has access to joint angles, and its kinematics is known. Joint torques, however, are not always accessible in some robotic hands. Hence, we address both cases in which joint torques can and cannot be measured. 

Let $\theta_{f_i}$ and $\x{f_i}\in\mathbb{R}^n$ be the vectors of finger joint angles and end-tip position of finger $i$. The forward-kinematics of the hand is given by
\begin{equation}
    \x{} = \lambda(\theta)
    \label{eq:hand_kinematics}
\end{equation}
where $\theta=(\theta_{f_1}^T,\ldots,\theta_{f_n}^T)^T$. Vector $\x{}=(\x{f_1}^T,\ldots,\x{f_n}^T)^T\in\mathbb{R}^{3n}$ is the concatenation of the finger tip coordinates relative to some global frame $\ssl{H}$ on the hand. 
Furthermore, let $\tau_{f_i}$ be the vector of joint torques of finger $i$. Given joint torques vector $\tau=(\tau_{f_1}^T,\ldots,\tau_{f_n}^T)^T$ for the hand, the forces exerted by the fingers on the objects at the contacts are given by
\begin{equation}
    \ve{f}_c = \left(J_h(\theta)^\dagger\right)^{T}\tau
    \label{eq:contact_forces}
\end{equation}
where $J_h(\theta)$ is the hand Jacobian, $J_h(\theta)^\dagger$ is the pseudo-inverse of $J_h(\theta)$ and $\ve{f}_c=(\ve{f}_{f_1}^T,\ldots,\ve{f}_{f_n}^T)^T\in\mathbb{R}^{3n}$ is the concatenation vector of the forces \citep{Murray1994}. The force vectors are also expressed relative to $\ssl{H}$ and can easily be transformed to the vector normals according to $\ve{n}_i=\ve{f}_{f_i}/\|\ve{f}_{f_i}\|\in\ssl{N}$. The result is a representation of the current grasp $\ssl{G}=\{\ssl{P},\ssl{N}\}$ or $\ssl{G}=\{\ssl{P}\}$, when joint torques can and cannot be accessed, respectively. A grasp can now be mapped to the parameterization vector $\ve{q}$ and used as an input to the classifier to identify the object.


To generate a training set, we collect samples from CAD models of the objects. A model of object $\ssl{O}_l$ can be defined by a set of tuples $\ssl{F}_l=\{(\x{1}, 
\ve{n}_1), (\x{2},\ve{n}_2),\ldots\}$ defining points and corresponding normals of the object surface. The number of contacts $n$ is defined by the number of fingers in the available robotic hand. We note, however, that grasp samples are collected by only picking contact locations on the objects independent of the kinematics of some hand. For each object, we collect $N$ samples of $n$-finger grasps of the form $\ssl{G}_i$ by sampling uniformly at random of $n$ tuples from $\ssl{F}_l$. Each sample is then mapped to a parameterization vector through $\ve{q}_i=\Phi_n(\ssl{G}_i)$. The vector $\ve{q}_i$ is labeled $o_i\in\{1,\ldots,m\}$ corresponding to object $\ssl{O}_{o_i}$. To enable tactile-like recognition, we do not require force-closure grasps but only that all fingertips are on the surface of the object. Thus, we may sample both force-closure and non-force-closure grasps. It is possible, as will be shown in experiments, to add noise to the CAD models in order to make the classifier robust to inaccurate models and finger-tip measurement errors. 

The outcome of the above process is a set of observed data 
$Q_M=\{(\ve{q}_1,o_1),\ldots,(\ve{q}_M,o_M)\}$, where $M=mN$. 
Dataset $Q_M$ can now be used to train a classifier for a set of objects, i.e., train map $h:\ssl{Q}\to[0,1]^m$ such that a class probability distribution $\ve{p}=\{p_1,\ldots p_m\}$, where $\sum_m p_i=1$, is computed by $\ve{p} = h(\ve{q})$. Once a classifier is trained and validated, the identification of an object can be performed by sampling finger positions on the object. However, the parameterization described above only represents a grasp and does not include the entire information of an object. Furthermore, measuring finger joint torques is not always possible or accurate enough such that information of a grasp can be partial (two different grasps can have the same finger positions $\ssl{P}$). Therefore, a classifier, that is trained with complete information of a grasp or solely on finger position representation, may yield object recognition with low certainty. Nevertheless, this could be improved by sequentially drawing more samples.

While common classification tasks rely only on one sample for prediction, in this case, we may acquire additional samples while being certain that they originate from the same class. Consider grasp vectors arriving sequentially $\{\ve{q}_1,\ve{q}_2,\ldots\}$ in real-time by sampling from a given query object whose class is unknown. It is required to estimate conditional probability for class $\ssl{O}_i$ given $k$ samples, i.e., $P_k(\ssl{O}_i|\ve{q}_1,\ldots\ve{q}_k)$. We propose to use either an iterative process with any trained classifier or sequential Bayesian updating. An illustration of the process can be seen in Figure \ref{fig:scheme}.

\subsection{Iterative classification}
\label{sec:ic}

Proposed by \cite{Kahanowich2021}, we track the scores of the classes based on the predictions for each sample provided by any chosen classifier. An iterative classification (IC) process to do so is described in Algorithm \ref{alg:iterative}. The algorithm maintains a vector $\ve{s}=(s_1,...,s_m)$ of cumulative scores for the classes. In each iteration, a grasp is sampled and parameterized, followed by acquiring a class probability distribution $\ve{p}$ using function $h$. Function $h$ can be any trained classifier that outputs a class probability distribution. The highest probability $p_i\in\ve{p}$ is the iteration score for class $i$ and is added to $s_i$. This process is repeated until the normalized cumulative score $\hat{s}_{max}=\max(\hat{\ve{s}})$, where $\hat{s}_j=s_j/\left(\sum_i{s}_i\right)\in\hat{\ve{s}}$, for some class reaches above a lower bound $\lambda_s\in[0,1]$. Normally, $\lambda_s$ will be chosen to be around $0.85-0.98$ for efficient performance. One may view $\hat{s}_j\in\ve{s}$ as the probability approximation for grasping $\mathcal{O}_j$ after some number $k$ of iterations, i.e., $\hat{s}_j\approx P_k(\mathcal{O}_j|\ve{q}_1,\ldots\ve{q}_k)$.

It is also possible to fully accumulate all class scores by updating $\ve{s}$ with all iteration probabilities, i.e., replace lines 6-7 in Algorithm \ref{alg:iterative} with $\ve{s}\leftarrow\ve{s}+\ve{p}$. However, this will result in an excessive number of iterations for ill-trained classifiers while requiring a carefully tuned $\lambda_s$. Also, preliminary results show that for a more accurate classifier, this would provide only marginal accuracy improvement with a higher number of iterations.

\begin{algorithm}
    \caption{$\mathtt{iterative\_classification} (\lambda_s)$}
    \label{alg:iterative}
    \SetAlgoLined
    Initiate elements of $\ve{s}=(s_1,...,s_m)$ to $0$\; 
    \Repeat{$\hat{s}_{max} > \lambda_s$}{
    Sample grasp $\ssl{G}$\;
    $\ve{q}\leftarrow\Phi_n(\ssl{G})$\;
    $\ve{p}\leftarrow h(\ve{q})$\; 
    $i\leftarrow arg\max(\ve{p})$\;
    $s_i\leftarrow s_i + p_i$\;
    $o\leftarrow arg\max(\ve{s})$\;
    \If{first iteration}{
        $\hat{s}_{max}\leftarrow s_o$\;}
    \Else{
        $\hat{s}_{max}\leftarrow s_o/\left(\sum_i{s}_i\right)$\;}}
    \Return{$o$}  \tcc*{return class index}
\end{algorithm}

Let $P_h(l=j|\mathcal{O}_i)$ be the probability for classifier $h$ to assign label $j$ to a grasped object $\mathcal{O}_i$ such that 
\begin{equation}
    \sum_{j=1}^m P_h(l=j|\mathcal{O}_i)=1.
\end{equation}
\begin{mydef}
\label{def:sufficient_classifier}
Classifier $h$ is sufficient if it satisfies
\begin{equation}
    P_h(l=i|\mathcal{O}_i)>P_h(l=j|\mathcal{O}_i)
    \label{eq:P>P}
\end{equation}
for any $i,j\in\{1,...,m\}$ and $j\neq i$. 
\end{mydef}
\noindent The above definition states that a sufficient classifier is more likely to correctly identify all objects in the set. Naturally, higher values of $P_h(l=i|\mathcal{O}_i)$ for all $i\in\{1,\ldots,m\}$ mean a more accurate classifier. A classifier may not be sufficient due to limited data, incompatible classification algorithm or non-optimal NN architecture. It is worth mentioning that, in practice, the $(j,i)$ component of a classifier's confusion matrix is an approximation of $P_h(l=j|\mathcal{O}_i)$.

\begin{thm}
Given a sufficient classifier $h$, it must be that the expected normalized cumulative score $\mathbb{E}(\hat{s}_i|\mathcal{O}_i)$ satisfies
\begin{equation}
    \mathbb{E}(\hat{s}_i|\mathcal{O}_i)>\mathbb{E}(\hat{s}_j|\mathcal{O}_i),~~ \forall j\neq i,
    \label{eq:s>s}
\end{equation}
when using Algorithm \ref{alg:iterative}.
\end{thm}
\begin{proof}
Let us assume, for now, that prediction probability of incorrect predictions tends to be lower than the prediction probability for correct ones \citep{Hendrycks2017}. Hence, given $p_{max}=\max(\ve{p})$, the expected value for $p_{max}$ when successfully classifying object $\mathcal{O}_i$ would be larger than an erroneous prediction. That is, statement 
\begin{equation}
    \mathbb{E}(p_{max}|l=i,\mathcal{O}_i)>\mathbb{E}(p_{max}|l=j,\mathcal{O}_i)
    \label{eq:expected_values_p}
\end{equation}
holds for any $j\neq i$. According to Algorithm $\ref{alg:iterative}$, vector $\ve{s}$ accumulates scores for class predictions with the increase of iterations. In addition, a score is given to $s_j$ only if label $l=j$ is assigned to the query object in a particular iteration. Hence, the expected normalized value $\hat{s}_j$ of component $s_j\in\ve{s}$ given object $\mathcal{O}_i$ is
\begin{equation}
    \mathbb{E}(\hat{s}_j|\mathcal{O}_i)=\mathbb{E}(p_{max}|l=j,\mathcal{O}_i)P_h(l=j|\mathcal{O}_i)
    \label{eq:expected_values_s}
\end{equation}
for any $j\in\{1,\dots,m\}$. From Definition \ref{def:sufficient_classifier} and \eqref{eq:expected_values_p}-\eqref{eq:expected_values_s}, it must be that \eqref{eq:s>s} is satisfied. Even with a more strict assumption, where $p_{max}$ for any prediction (erroneous or not) has a uniform distribution $p_{max}\sim U(\frac{1}{m},1)$ such that
\begin{equation}
     \mathbb{E}(p_{max}|l=i,\mathcal{O}_i)=\mathbb{E}(p_{max}|l=j,\mathcal{O}_i)=\frac{m+2}{2m},
     \label{eq:uniform_p_max}
\end{equation}
statement \eqref{eq:s>s} remains valid due to \eqref{eq:P>P} and \eqref{eq:expected_values_s}. 
\end{proof}


The above theorem implies that as long as a classifier satisfies Definition \ref{def:sufficient_classifier} and while holding object $\mathcal{O}_i$, the cumulative score $\hat{s}_i$ will increase and converge to an higher value than $\hat{s}_j$ ($j\neq i$), with the increase of iterations, i.e., $\hat{s}_i$ is more likely to be maximal and equal to $\hat{s}_{max}$. Hence, the certainty about the prediction will grow with the addition of more samples. In turn, this will result in continuous improvement of the classifiers success rate. Let us assume that some classifier is not sufficient such that $m_p\in\{0,1,\ldots,m\}$ is the number of grasped objects that does satisfy \eqref{eq:P>P}. When $m_p<m$, the classification success rate will increase only for these objects while declining for the remaining $m-m_p$ ones. This means that the  success rate upper limit for a certain classifier is \begin{equation}
\eta=\frac{m_p}{m} \times 100\%.
\label{eq:bound}
\end{equation} 
If a classifier is sufficient such that $m_p=m$, the total success rate would converge to 100\% with the increase of iterations. The convergence rate depends on the quality of the classifier, that is, on the accuracy $P_h(l=i|\mathcal{O}_i)$ and whether \eqref{eq:expected_values_p} is satisfied. Hence, the proper amount of iterations to reach some level of accuracy or certainty is not known beforehand. Consequently, we cannot set the termination criterion to some arbitrary number of iterations and Algorithm \ref{alg:iterative} terminates when reaching some certainty above a threshold $\lambda_s$.

\subsection{Sequential Bayesian Update}
\label{sec:sbu}

In an alternative approach, we consider Bayesian classification (BC). We assume a Markovian model where samples are mutually independent such that $P(\ve{q}_k|\ssl{O}_t, \ve{q}_1,\ldots,\ve{q}_{k-1})=P(\ve{q}_k|\ssl{O}_t)$. The joint probability density function  given $k$ samples is, therefore, expressed by the Bayes rule
\begin{equation}
    P_k(\ssl{O}_t|\ve{q}_1,\ldots\ve{q}_k)\propto P(\ssl{O}_t) \prod_{i=1}^k P(\ve{q}_i|\ssl{O}_t),
\end{equation}
where $P(\ssl{O}_t)$ is the prior distribution before observing the samples \citep{Kay1993}. Probability $P(\ve{q}_i|\ssl{O}_t)$ is the likelihood of observing sample $\ve{q}_i$ given object $\ssl{O}_t$. The likelihood can be approximated by using the Kernel Density Estimator (KDE) while observing the data for each object \citep{parzen1962}. KDE is a non-parametric method to learn and estimate a probabilistic density function automatically from data. Given the $M_t$ independent and identically distributed training samples $\{\ve{q}^{(t)}_1,\ldots,\ve{q}^{(t)}_{M_t}\}\subset Q_M$ corresponding to object $\ssl{O}_t$, KDE is formally defined as
\begin{equation}
    P(\ve{q}_i|\ssl{O}_t)=\frac{1}{M_t}\sum_{j=1}^{M_t} K_\sigma\left(\ve{q}_i - \ve{q}^{(t)}_j\right)
\end{equation}
where $K_\sigma:\mathbb{R}^w\to\mathbb{R}$ is a smooth function, termed the \textit{kernel} function, with bandwidth $\sigma>0$. A common choice, as in this work, is a Gaussian kernel given by
\begin{equation}
    K_\sigma(\ve{x}) = \frac{1}{\sqrt{2\pi}}e^{\left(-\frac{1}{2\sigma^2}\ve{x}^2\right)}.
\end{equation}

Algorithm \ref{alg:bayesian_classification} describes the iterative process with some resemblance to Algorithm \ref{alg:iterative}. For each sampled and parameterized grasp $\ve{q}$, the likelihood $P(\ve{q}|\ssl{O}_i)$ for each object ($i=1,\ldots,m$) is computed, followed by updating the posterior probability vector $\ve{p}$. This process is repeated until the posterior probability for some class reaches above a lower bound $\lambda_p\in[0,1]$. Here also, $\lambda_p$ should be chosen to be around $0.85-0.98$ for efficient performance. The prior probability distribution for $\ve{p}_{prior}=\left(P(\ssl{O}_1),\ldots,P(\ssl{O}_m)\right)$ can either be chosen naively (Naive Prior - NP) to be $P(\ssl{O}_i) = \frac{1}{m}$ ($i\in\{1,\ldots,m\}$) or based on the initial prediction of some other classifier $h$ previously trained on the data (Initial Prior - IP). 

\begin{algorithm}
    \caption{$\mathtt{bayesian\_classification} (\lambda_p, \ve{p}_{prior})$}
    \label{alg:bayesian_classification}
    \SetAlgoLined
    Initiate $\ve{p}=(p_1,...,p_m)$ to $\ve{p}_{prior}$\;
    \Repeat{$p_{o} > \lambda_p$}{
    Sample grasp $\ssl{G}$\;
    $\ve{q}\leftarrow\Phi_n(\ssl{G})$\;
    \For{$i=1 \to m$}{
        Compute $P(\ve{q}|\ssl{O}_i)$ using KDE\;
        ${p}_i\leftarrow {p}_i \cdot P(\ve{q}|\ssl{O}_i)$\;
    }
    $\ve{p}\leftarrow \ve{p}/\left(\sum_i{p}_i\right)$  \tcc*[r]{Norm. to sum. 1}
    $o\leftarrow arg\max(\ve{p})$\;
    }
    \Return{$o$}  \tcc*{return class index}
\end{algorithm}

\subsection{Scale invariant grasp representation}
\label{sec:scale}

While the grasp parameterization described in Section \ref{sec:grasp_rep} is injective and independent of any coordinate system, it remains dependent of object size. Two objects of the same shape but with different size will have a non-equal parameterization vector resulting in a non-scalable classifier. Hence, a model trained on a set of objects will not be able to accurately classify a scaled version of the same objects. A solution for this would be to normalize the grasp. Hence, we first define the scaling of an object.
\begin{mydef}
Object $\ssl{O}_i$ is uniformly scaled by $\xi>0~$ to $~\ssl{O}_i^{(\xi)}$ if every point $\ve{p}_k\in\ssl{O}_i$ is mapped to $\xi\cdot\ve{p}_k$, such that $\xi\cdot\ve{p}_k\in\ssl{O}_i^{(\xi)}$.
\end{mydef}
\noindent Throughout this paper, uniform scaling is referred as scaling. Next, we define the scaling of a grasp.
\begin{mydef}
Let $\ssl{G}_i$ be a grasp of object $\ssl{O}_i$ at contact points $\ssl{P}_i$. A scaled grasp $\ssl{G}_i^{(\xi)}$ corresponds to grasping of object $\ssl{O}_i^{(\xi)}$ at contact points $\xi\cdot\ssl{P}_i$.
\end{mydef}
\noindent The above definition addresses the scaling of the grasp polyhedron $\Theta(\ssl{P}_i)$ by scaling the vertices by $\xi$. The direction of the normals at the contact points do not change, i.e., scaling of grasp $\ssl{G}=\{\ssl{P},\ssl{N}\}$ by $\xi$ results in $\ssl{G}^{(\xi)}=\{\xi\cdot\ssl{P},\ssl{N}\}$. We note that scaling of a grasp can also be accompanied with rotation and translation of $\ssl{G}$ while not affecting the parameterization as described in Theorem \ref{thm:frame_invariance}. Without loss of generality, we omit possible transformation when scaling a grasp for simplicity of presentation in the remaining part of the paper.

Let $a_{vu}$ be the surface area of facet $u$ on polyhedron $\Theta(\ssl{P}_v)$ ($u=1,...,U$ where $U$ is the number of facets of $\Theta(\ssl{P}_v)$), we define $A_v=\sqrt{\sum_{u=1}^Ua_{vu}}$ to be the square root of the total surface area of polyhedron $\Theta(\ssl{P}_v)$. We now argue that the scaling of two grasps corresponding to scaled objects, each by its own squared total surface area, will result in a scale invariant grasp representation.
\begin{thm}
\label{thm:equal_phi}
Given grasps $\ssl{G}_i$ and $\ssl{G}_j$ of objects $\ssl{O}_i$ and $\ssl{O}_j$, respectively, where some unknown scaling factor $\xi$ exists such that $\ssl{O}_j\equiv\ssl{O}_i^{(\xi)}$ and $\ssl{G}_j=\ssl{G}_i^{(\xi)}$. By using scaling factors $A_i^{-1}$ and $A_j^{-1}$, it must be that 
\begin{equation}
    \Phi_n\left(\ssl{G}_i^{\left(A_i^{-1}\right)}\right)=\Phi_n\left(\ssl{G}_j^{\left(A_j^{-1}\right)}\right)
\end{equation}
for any number of fingers $n$.
\end{thm} 
\begin{proof}
See Appendix B.
\end{proof}
Theorem \ref{thm:equal_phi} states that a scale invariant training set can be generated by scaling each grasp of an object with the total surface area of the formed polyhedron. Then, a trained classifier can identify the scaled objects corresponding to the original ones. The scale invariance property could remove the need for recollection and retraining of a classifier for a new object that is a scaled version of an already known one. In addition, scaling can enable generalization for geometry classification as will be shown in the results.


\subsection{Embedded $z$-finger grasps}
\label{sec:z-grasps}

While sampling an $n$-finger grasp, one may consider it as a set of possible $z$-finger grasps, where $z<n$. In an $n$-finger grasp, there are $c_{n,z}$ combinations of $z$-finger grasps given by
\begin{equation}
    c_{n,z}={n \choose z}=\frac{n!}{z!(n-z)!}.
\end{equation}
Therefore, we may use a classifier trained over $z$-finger grasps to better exploit the information given by a single $n$-finger sample. Hence, potentially reduce the overall number of physical grasp samples. Algorithms \ref{alg:iterative_3_in_n} and \ref{alg:bayesian_classification_3_in_n} present the modifications of the IC and BC algorithms, respectively, to use $z$-finger grasps. 
For every grasp sample $\ssl{G}$, all possible combinations of $z$-finger grasps $\ssl{G}^z_j$ ($j\in\{1,...,c_{n,z}\}$) are used to update classification certainty as described in Sections \ref{sec:ic} and \ref{sec:sbu}.

\begin{algorithm}
    \caption{$\mathtt{iterative\_z\_classification} (\lambda_s)$}
    \label{alg:iterative_3_in_n}
    \SetAlgoLined
    Initiate elements of $\ve{s}=(s_1,...,s_m)$ to $0$\; 
    \While{$True$}{
        Sample grasp $\ssl{G}$\;
        $\{\ssl{G}^z_1,\ldots,\ssl{G}^z_{c_{n,z}}\}\leftarrow$ all $z$-finger combinations in $\ssl{G}^n$\;
        \For{$j = 1 \to c_{n,z}$}{
            $\ve{q}\leftarrow\Phi_z(\ssl{G}^z_j)$\;
            $\ve{p}\leftarrow h(\ve{q})$\; 
            $i\leftarrow arg\max(\ve{p})$\;
            $s_i\leftarrow s_i + p_i$\;
            $o\leftarrow arg\max(\ve{s})$\;
            \If{first iteration}{
                $\hat{s}_{max}\leftarrow s_o$\;}
            \Else{
                $\hat{s}_{max}\leftarrow s_o/\left(\sum_i{s}_i\right)$\;}
            \If{$\hat{s}_{max} > \lambda_s$}{
                \Return{$o$}  \tcc*{ret. class index}
            }
        }
    }
    
\end{algorithm}

\begin{algorithm}
    \caption{$\mathtt{bayesian\_}z\mathtt{\_classification} (\lambda_p, \ve{p}_{prior})$}
    \label{alg:bayesian_classification_3_in_n}
    \SetAlgoLined
    Initiate $\ve{p}=(p_1,...,p_m)$ to $\ve{p}_{prior}$\;
    \While{$True$}{
        Sample grasp $\ssl{G}^n$\;
        $\{\ssl{G}^z_1,\ldots,\ssl{G}^z_{c_{n,z}}\}\leftarrow$ all $z$-finger combinations in $\ssl{G}^n$\;
        \For{$j = 1 \to c_{n,z}$}{
            $\ve{q}\leftarrow\Phi_z(\ssl{G}^z_j)$\;
            \For{$i=1 \to m$}{
                Compute $P(\ve{q}|\ssl{O}_i)$ using KDE\;
                ${p}_i\leftarrow {p}_i \cdot P(\ve{q}|\ssl{O}_i)$\;
            }
            $\ve{p}\leftarrow \ve{p}/\left(\sum_i{p}_i\right)$\;  
            $o\leftarrow arg\max(\ve{p})$\;
            \If{$p_{o} > \lambda_p$}{
                \Return{$o$}  \tcc*{ret. class index}
            }
        }
    }
\end{algorithm}

Considering a lower number of fingers than the hand has can also be useful in exploiting incomplete grasps. An $n$-finger robotic hand may not always be able to make contact of all fingers with the object. Feedback for lack of finger contact can be acquired from joint torques or some low-level haptics, if available. In such case, the $z<n$ fingers that achieved contact can also provide viable information without the need of a new grasp. Hence, classification using either IC or BC can be done with a sampling sequence of an heterogeneous number of fingers, e.g., grasp sequence $\{\ssl{G}^n, \ssl{G}^{z_1}, \ssl{G}^{z_2}, \ssl{G}^n, \ssl{G}^{z_3},\ldots\}$ where $z_j<n$ may be acquired. Each $\ssl{G}^j$ is parameterized and classified based on the corresponding trained model.


\section{Experimental Results and Analysis}
\label{sec:results}

In this section, we test and analyse the proposed method over a set of distinct objects while not considering the kinematics of the robotic hand. Experiments with robotic hands are presented in Section \ref{sec:experiments}. We have picked eight objects from the KIT object models database \citep{Kasper2012} shown in Figure \ref{fig:KITobjects}. The KIT database is a set of objects scanned with a 3D digitizer, resulting in a triangular mesh of 800 faces for each object. The eight objects were selected because they offer a wide diversity of shapes while some have irregular ones (the \textit{ToyCarYelloq} and the \textit{CatLying}), and others have relatively similar shapes that are challenging to make a distinction (\textit{Shampoo} and \textit{ShowerGel}). We first present the process of data collection and training of various classifiers over the set of CAD models. In Section \ref{sec:classification_naive}, we present results for object recognition using IC and BC. Further, we test the scalability property and the use of $z$-finger grasps. Finally, we demonstrate the ability to classify types of geometries over a larger set of objects.

\begin{table}[ht]
    \centering
    \begin{tabular}{cc}
        \includegraphics[width=0.51\linewidth]{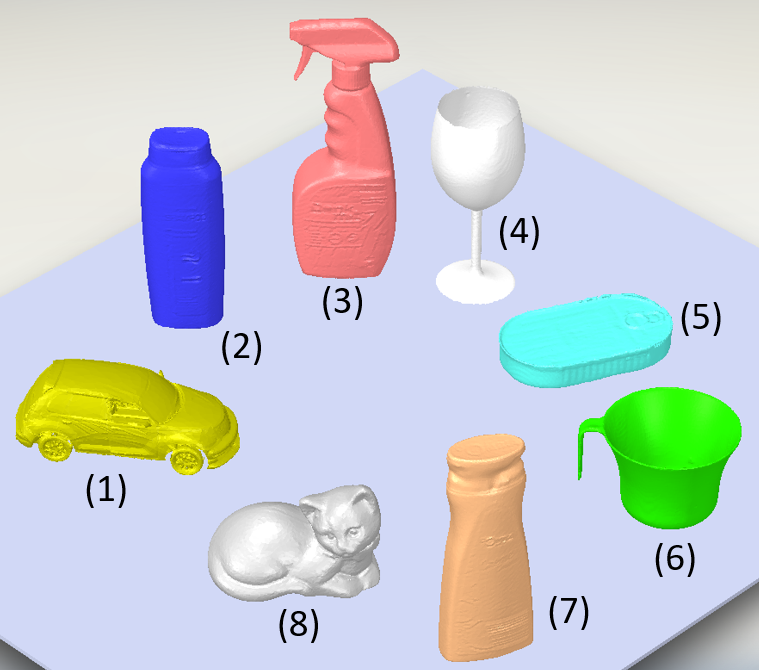} & 
    \begin{tabular}[b]{|l|l|}\hline
           & Object  \\\hline
        (1)  & ToyCarYelloq \\
        (2)  & Shampoo      \\
        (3)  & SprayFlask   \\
        (4)  & Wineglass    \\
        (5)  & HeringTin    \\
        (6)  & GreenCup     \\
        (7)  & ShowerGel    \\
        (8)  & CatLying     \\\hline
        \end{tabular} \\
    \end{tabular}
    \captionof{figure}{Eight objects from the KIT object models database used in the experiments.}
    \label{fig:KITobjects}
\end{table}

\subsection{Data collection and model training}
\label{sec:data_collection}

\begin{table*}[]
\centering
\caption{Success rate comparison for different classifiers}
\label{tb:classifiers}
\begin{tabular}{|lc||cc|cc|cc|}\hline
Num. fingers         && \multicolumn{2}{c|}{3} & \multicolumn{2}{c|}{4} & \multicolumn{2}{c|}{5} \\
normals              && w/   & w/o       & w/        & w/o       & w/     & w/o    \\\hline
\multicolumn{1}{|l|}{Nearest Neighbors}  &\multirow{7}{*}{\small \%}
                                       & 70.8 & 45.6 & 59.1 & 56.1 & 62.2 & 61.8 \\
\multicolumn{1}{|l|}{Linear SVM}      && 41.2 & 41.0 & 50.0 & 49.6 & 55.6 & 55.7 \\
\multicolumn{1}{|l|}{RBF SVM}         && 76.9 & 46.9 & 47.6 & 57.8 & 37.3 & 63.5 \\
\multicolumn{1}{|l|}{Gaussian Process}              && 69.5 & 42.7 & 59.0 & 57.3 & 64.7 & 63.8 \\
\multicolumn{1}{|l|}{Random Forests}   && 66.8 & 47.3 & 58.7 & 57.7 & 64.6 & 63.3 \\
\multicolumn{1}{|l|}{Neural Network}              && \cellcolor[HTML]{C0C0C0}86.0 & \cellcolor[HTML]{C0C0C0}47.0 & \cellcolor[HTML]{C0C0C0}80.7 & \cellcolor[HTML]{C0C0C0}62.4 & \cellcolor[HTML]{C0C0C0}83.7 & \cellcolor[HTML]{C0C0C0}71.1 \\
\multicolumn{1}{|l|}{Decision Trees}      && 61.9 & 45.7 & 58.2 & 54.4 & 61.8 & 61.0 \\
\hline
\end{tabular}
\end{table*}

The KIT objects are represented as triangular meshes. In the context of possible contact locations, each mesh is converted to a set $\ssl{F}_l$ of contact locations (centers of the triangles) and the normals at the contacts (normals to the triangles) directed inwards. We generate a training set as described in Section \ref{sec:obj_classification}. An $n$-finger grasp is sampled by randomly selecting $n$ elements from $\ssl{F}_l$ (with or without normals). A training set is acquired by sampling and parameterizing $N=200,000$ $n$-finger grasps for each object, and labeling them. We note that 15\% of the training data was used for validation during training. The testing phase, from which we present results, was performed after training and done by directly sampling grasps from the objects. The average generation time of data for each object is 4.72 minutes computed on an Intel-Core i7-9700 Ubuntu machine with 16 GB of RAM. 

For IC and for the prior of BC-IP, we have trained and tested different classifiers, including: Nearest-Neighbors, Support Vector Machines (SVM), Gaussian Processes, Random Forests, Neural-Network and Decision Trees. The SVM classifier was trained using linear and Radial Basis Function (RBF) kernels. The NN architectures were formed and optimized individually for different input dimensions, depending on the number of fingers and whether the normals were included. For example, a 4-finger grasp with normals, that yields an input dimension of 14, used a deep NN of seven hidden layers of 500 width. A 9-finger grasp with normals and parameterization vector of dimension 39 reached the best results with a NN of eight layers and width of 289 neurons. The networks all used a Rectified Linear Unit (ReLU) activation function and were trained with the back-propagation algorithm.

Table \ref{tb:classifiers} reports the classification success rate of 1,000 trials per object with 3-, 4- and 5-finger grasps directly sampled from the objects. We note that these are the rates for a one-shot without any iterative classification, i.e., classification based on a single grasp sample. Most of the classification techniques performed poorly, while the NN outperforms. However, it is clear that a single grasp sample generally cannot be used to reliably classify an object. As discussed in Section \ref{sec:obj_classification}, one grasp and its parameterization does not embed the entire information of the object and is not sufficient. Hence, additional grasps can be used to increase the certainty about the object. We use the NN classifier in IC and a prior for BC-IP in further experiments. Results for iterative object recognition are presented next.

\begin{figure*}[ht]
    \centering
    \setlength{\tabcolsep}{4pt} 
    \resizebox{\textwidth}{!}{%
    \begin{tabular}{cccc}
        \includegraphics[width=0.22\textwidth]{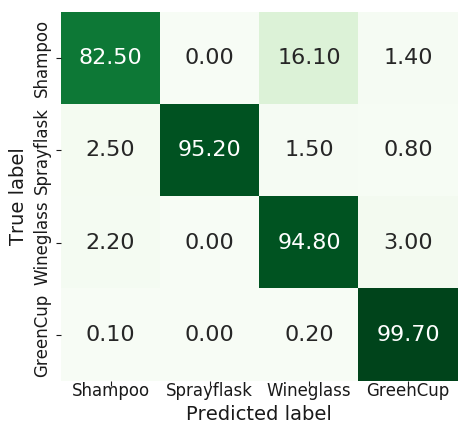} & \includegraphics[width=0.22\textwidth]{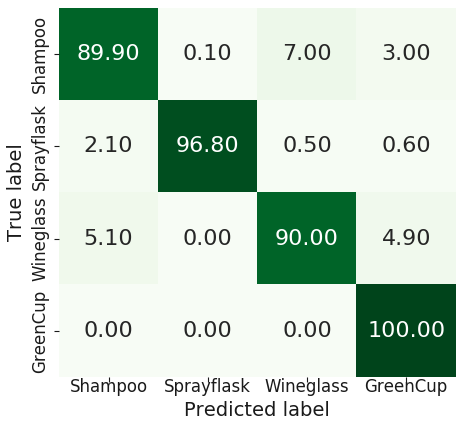} & \includegraphics[width=0.22\textwidth]{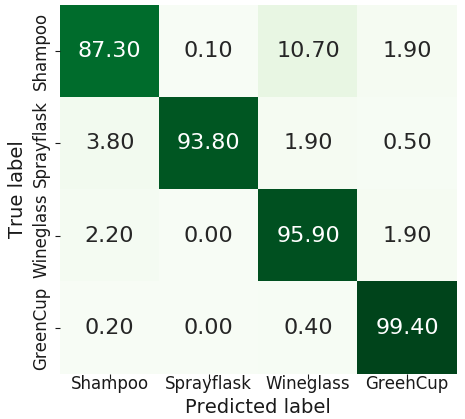} & \includegraphics[width=0.26\textwidth]{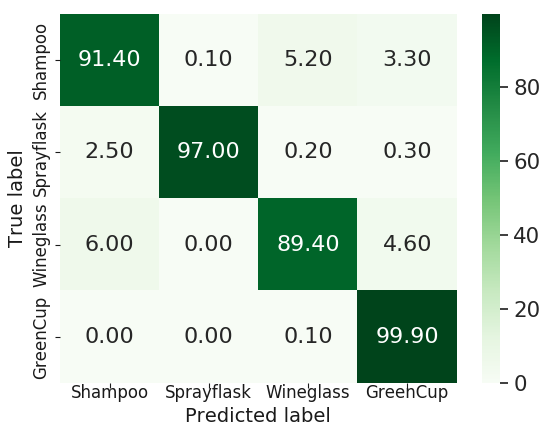} \\
        (a) & (b) & (c) & (d) \\
    \end{tabular}}
    \caption{Confusion matrices for classification of four objects using BC-NP. The results are for 1,000 query trials per object. Classification with a 4-finger grasp (a) with and (b) without the normals at the contact points, and a 5-finger grasp (c) with and (d) without the normals. Average number of re-sampling iterations is (a) 4.87, (b) 6.52, (c) 3.95 and (d) 5.21, for $\lambda_p=0.85$.}
    \label{fig:CM_BC_NP}
\end{figure*}

\begin{figure*}[ht]
    \centering
    \setlength{\tabcolsep}{3pt} 
    \resizebox{\textwidth}{!}{%
    \begin{tabular}{cccc}
        \includegraphics[width=0.23\textwidth]{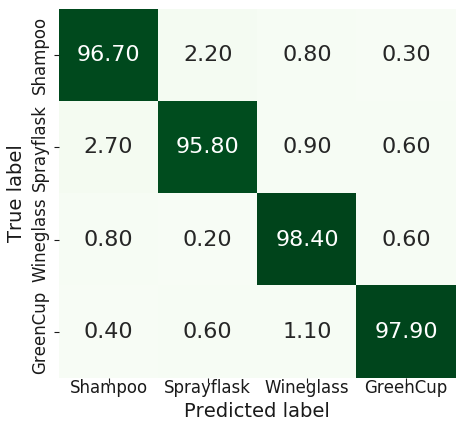} & \includegraphics[width=0.23\textwidth]{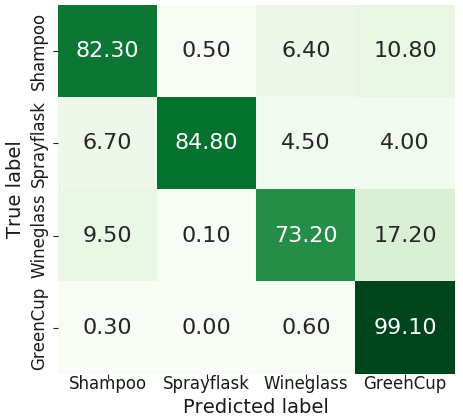} & \includegraphics[width=0.23\textwidth]{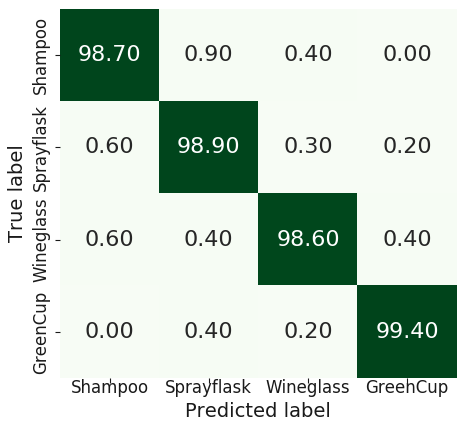} & \includegraphics[width=0.27\textwidth]{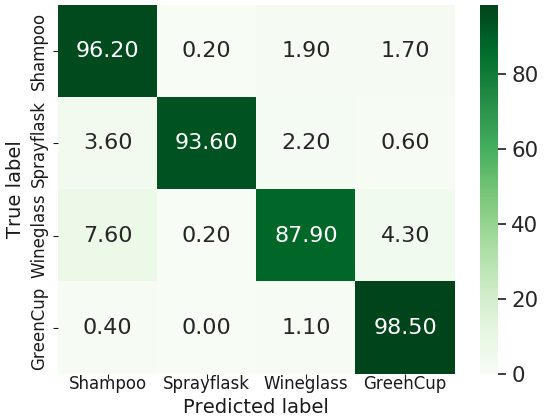} \\
        (a) & (b) & (c) & (d) \\
    \end{tabular}}
    \caption{Confusion matrices for classification of four objects using IC. The results are for 1,000 query trials per object. Classification with a 4-finger grasp (a) with and (b) without the normals at the contact points, and a 5-finger grasp (c) with and (d) without the normals. Average number of re-sampling iterations is (a) 4.35, (b) 3.6, (c) 4.67 and (d) 3.97, for $\lambda_p=0.85$.}
    \label{fig:CM_IC}
\end{figure*}

\subsection{Object Recognition}
\label{sec:classification_naive}

In this section, we experimentally show and analyse the success rate and the required number of samples to classify an object through contact locations, with and without including the normals at the contacts. For both IC and BC, we have manually tested different values for $\lambda_s$ and $\lambda_p$. The value $\lambda_s=\lambda_p=0.85$ has yielded sufficient results, both in terms of success rate and minimum number of grasp samples. Surely this number can be higher to improve the success rate at the cost of more grasp samples. We use this value in the following experiments while later showing an analysis of the success rate with regards to the number of grasp samples.

For a clear presentation, we first show results for classifying the four object: \textit{Shampoo}, \textit{SprayFlask}, \textit{Wineglass} and \textit{GreenCup}. Figure \ref{fig:CM_BC_NP} shows confusion matrices for classification with BC-NP for 4- and 5-finger grasps with and without including the normals. Not using the normals requires an average of about one more sample in order to reach the lower bound probability. Similarly, Figure \ref{fig:CM_IC} shows results for classification with IC. Here, not using the normals required less iterations with the cost of a lower success rate. Complete results for 3 to 10 finger grasps with the eight objects are shown in Figures \ref{fig:num_fingers_results_wN} and \ref{fig:num_fingers_results_woN}, with and without including the normals, respectively. The figures report results for the classification success rate and the required number of samples. Table \ref{tb:mean_samples} shows the average number of iterations including corresponding standard-deviations. When not considering the normals, 2-finger grasps are included as discussed in Section \ref{sec:grasp_rep}. 
Generally, all three methods (i.e., IC, BC-NP and BC-IP) acquire high success rate with a slight advantage for IC. However, IC performs poorly when not including normals and with a low number of fingers (2-3), as seen in Figure \ref{fig:num_fingers_results_woN}. NN-based classifiers trained over grasps with a low number of fingers and no normals failed to satisfy condition \eqref{eq:P>P}. Hence, they exhibited poor results for IC. In addition, the use of two fingers produced relatively high success rate with BC, but with the high cost of too many grasp samples. Overall, using more fingers provides more information that decreases the need for additional grasp samples. This motivates the use of $z$-finger grasps to be tested in the next section. 

When comparing between BC-NP and BC-IP, the former reaches better success rate while the latter requires less samples. While the BC-IP requires less samples in cases where the prior for a certain object is high enough, the prior can also wrongly bias the classification when the probabilities for all objects are rather low. Also, the accuracy of BC-NP decreases, when considering normals and with the addition of more fingers, due to the difficulty to cope with the increase of dimensionality. In general, IC and BC perform better in queries with and without normals, respectively. Figure \ref{fig:iterations_objs4} presents the success rate for 4-finger grasps with regards to the number of grasps samples required. The results shown in the figure will be addressed in more detail in the next section. Figure \ref{fig:cm_20} demonstrates the ability to classify 20 objects from the KIT dataset (including the eight previously used). The figure presents a confusion matrix of object recognition for 5-finger grasps without including the normals and when using BC-NP. The results show that accurate object recognition can be performed for many objects without normals. 

\begin{figure}[ht]
\centering
\includegraphics[width=\linewidth]{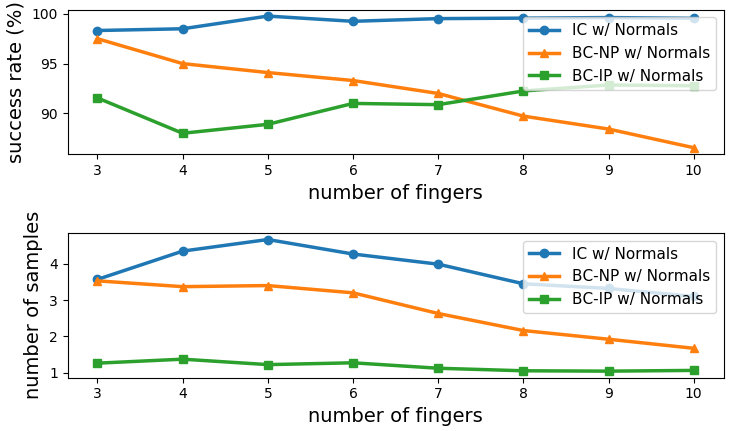}
\caption{Success rate (top) and average number of samples (bottom) with regards to the number of fingers used in the sampling while using the normals at the contacts. Plots show results for IC, BC-NP and BC-IP.}
\label{fig:num_fingers_results_wN}
\end{figure}
\begin{figure}[ht]
\centering
\includegraphics[width=\linewidth]{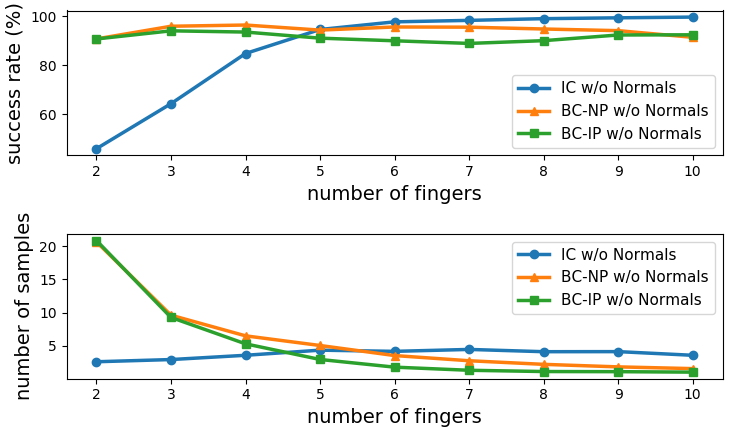}
\caption{Success rate (top) and average number of samples (bottom) with regards to the number of fingers used in the sampling while \textbf{not} using the normals at the contacts. Plots show results for IC, BC-NP and BC-IP.}
\label{fig:num_fingers_results_woN}
\end{figure}

We next observe performance of the methods with regards to size of the training data. Results for success rate and average number of iterations with regards to the number of training points can be seen in Figure \ref{fig:data_usage}. The results show that less than half of the data is required to achieve sufficient accuracy. Nevertheless, the data is collected relatively fast (a few minutes) through a computation process without the need for sampling from the physical objects. Hence, collecting excessive data is relatively cheap. In addition, the results show the benefit of improving accuracy using IC when comparing between one-shot classification accuracy (dotted blue curves) and IC (solid blue curves).


\begin{figure*}[h]
\centering
\includegraphics[width=0.9\linewidth]{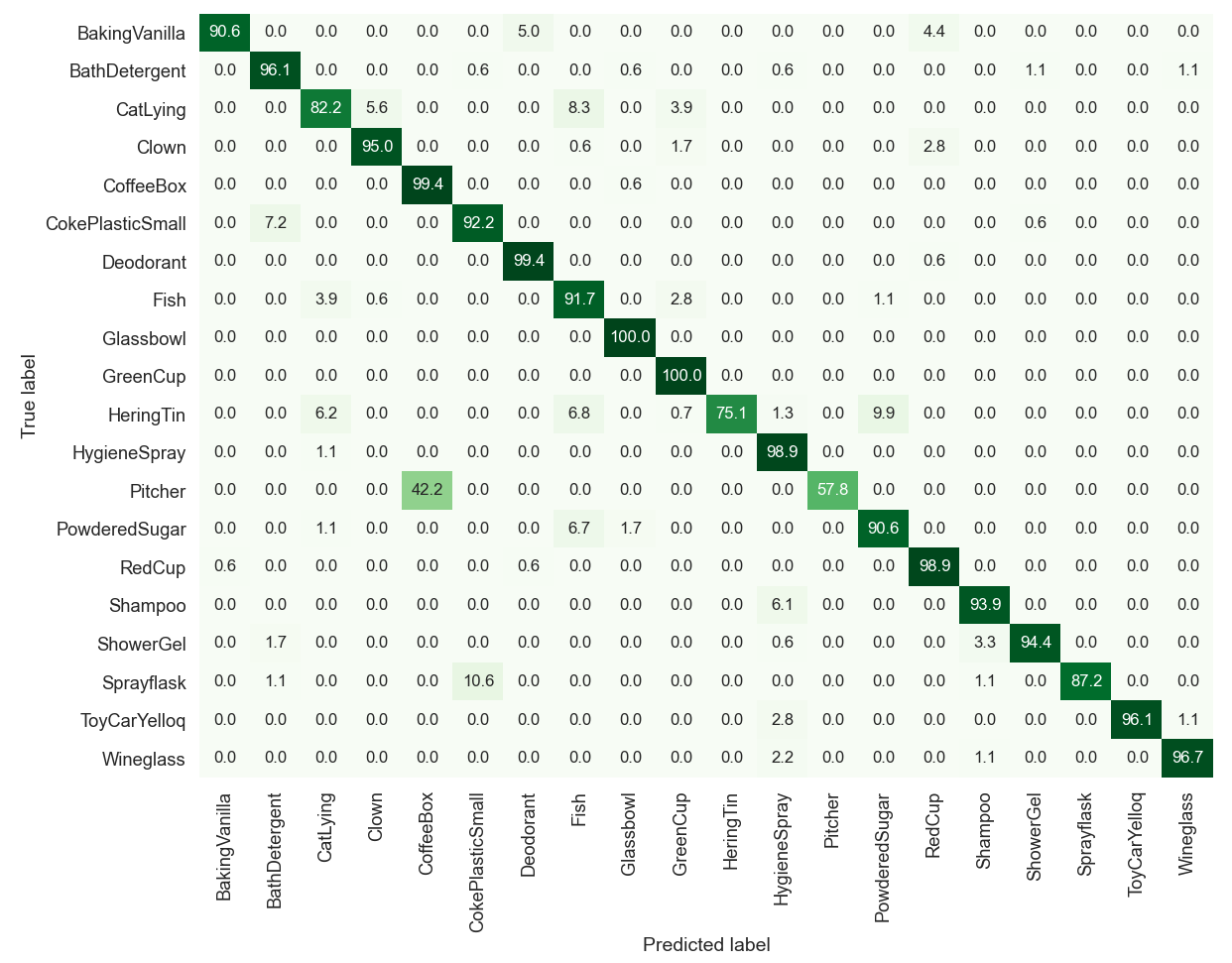}
\caption{Confusion matrix for classification of 20 objects using BC-NP with $5$-finger grasps and without normals.}
\label{fig:cm_20}
\end{figure*}

\begin{table}[]
\centering
\caption{Number (mean and std.) of grasp samples}
\label{tb:mean_samples}
\begin{tabular}{|l|cc|cc|cc|}\hline
\# fingers         & \multicolumn{2}{c|}{3} & \multicolumn{2}{c|}{4} & \multicolumn{2}{c|}{5} \\
normals              & w/   & w/o       & w/        & w/o       & w/     & w/o    \\\hline
& \multicolumn{6}{c|}{IC} \\\hline
\multicolumn{1}{|l|}{Mean}      & 3.57 & 2.96 & 4.35 & 3.60 & 4.67 & 4.37 \\
\multicolumn{1}{|l|}{Std.}         & 4.57 & 3.25 & 5.80 & 4.45 & 6.43 & 5.70 \\\hline
& \multicolumn{6}{c|}{BC-NP} \\\hline
\multicolumn{1}{|l|}{Mean}      & 3.53 & 9.62 & 3.37 & 6.52 & 3.40 & 5.07 \\
\multicolumn{1}{|l|}{Std.}         & 2.07 & 13.81 & 1.90 & 4.87 & 1.46 & 3.04 \\\hline
\end{tabular}
\end{table}

\begin{figure}
\centering
\includegraphics[width=\linewidth]{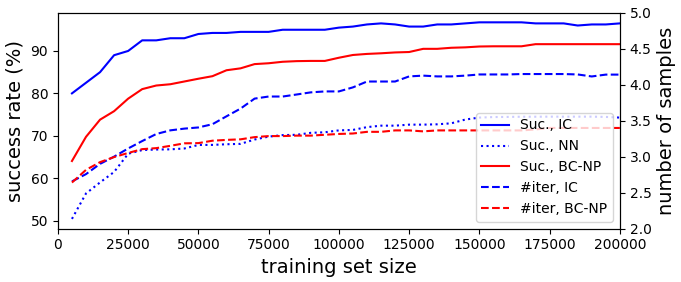}
\includegraphics[width=\linewidth]{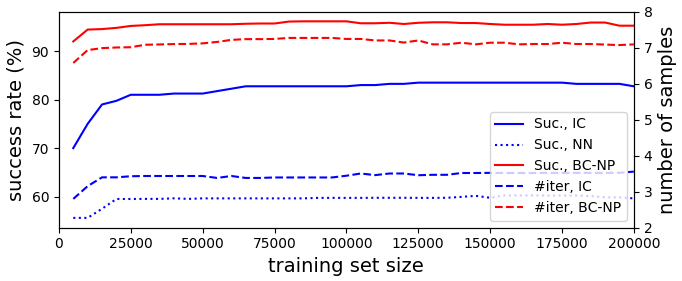}
\caption{Classification success rate (solid curves) and average number of iterations (dashed curves) with regards to the size of the training data for IC (blue) and BC-NP (red) while using (top) and not using (bottom) the normals at the contacts. Dotted curves are the success rates for one-shot classification of the NN classifiers prior to using IC.}
\label{fig:data_usage}
\end{figure}

\subsection{$z$-finger grasps}
In this Section, we evaluate performance improvement when physically sampling $n$-finger grasps and use classifiers trained over $z$-finger grasps ($z<n$), i.e., using algorithms \ref{alg:iterative_3_in_n} or \ref{alg:bayesian_classification_3_in_n}. Experiments were performed as previously discussed while using classifiers (either KDE for BC or NN for IC) trained over 3- and 4-finger grasps ($z=3$ or $z=4$, respectively). For $z=3$, we test $n=\{4,5,6,7\}$ finger grasps, having $c_{n,3}=\{4,10,20,35\}$ combinations of 3-finger grasps, respectively. Similarly, for $z=4$, we test $n=\{5,6,7\}$ finger grasps corresponding to $c_{n,4}=\{5,15,35\}$ grasp combinations, respectively. In practice, we randomly sample four combinations out of the $c_{n,z}$ as preliminary tests show that more do not provide significantly better results. Table \ref{tb:3_in_objects} presents the success rate and the average number of $n$-finger grasp samples taken when using IC and BC-NP. The table also presents the reduction in the average number of samples compared to the results for regular sampling presented in Section \ref{sec:classification_naive}. In all methods, the success rates when using the normals are similar to previous ones (Figure \ref{fig:num_fingers_results_wN}) while the average number of samples is significantly reduced by 40\%-70\%. However, when not using the normals and for IC,  reduction exists while the success rates are low. This is since the original NN classifiers for $n=3$ and $n=4$ performed poorly as seen in Table \ref{tb:classifiers} and Figure \ref{fig:num_fingers_results_woN}. Nevertheless, BC-NP yielded better success rates with some improvements in the required average number of samples in most cases. Conclusively, IC and BC with $z$-finger grasps provide best results for object recognition with and without normals, respectively.

Figure \ref{fig:iterations_objs4} presents classification success rates with regards to the number of grasp samples and for all the above methods. The results were generated while disregarding the bounds $\lambda_s$ and $\lambda_p$. Thus, we measured the success rate when a certain number of grasp samples is taken. Almost all methods are able to reach 100\% success rate with the increase of grasp samples. However, some require an excessive amount of samples which is not practical. Nevertheless, a small number of samples is required to achieve relatively high success rate. Evidently, the use of a $z$-finger model reduces the number of samples significantly when using the normals and slightly when not using normals. We note that IC with $z=3$ without normals converged to some success rate lower than 100\% since its NN classifier does not satisfy condition \eqref{eq:P>P} for all objects.

As described in Section \ref{sec:z-grasps}, $z$-finger grasps can also be used to exploit incomplete grasps when not all fingers are in contact. We test the performance of such case for $4$- and $5$-finger robotic hands. For each grasp instant, we assume that the number of fingers that may be in contact are given by the discrete probability distribution $p_n(z)=Pr(z=a)$. For $n=4$ and where $a=\{3,4\}$, the probability distribution is defined as 
\begin{equation}
    \label{eq:p4(z)}
    p_4(z)=\left\{
    \begin{array}{lll}
      Pr(z=3) & = & 0.4 \\
      Pr(z=4) & = & 0.6 \\
\end{array} \right..
\end{equation}
Similarly, when $n=5$ and $a=\{3,4,5\}$, we define the probability distribution as
\begin{equation}
    \label{eq:p5(z)}
    p_5(z)=\left\{
    \begin{array}{lll}
      Pr(z=3) & = & 0.2 \\
      Pr(z=4) & = & 0.3 \\
      Pr(z=5) & = & 0.5 \\
\end{array} \right..
\end{equation}
Grasps of less than three fingers are not taken in account and are re-sampled. Table \ref{tb:p(z)} presents classification results with the above probability distributions for IC and BC-NP. The results indicate that accuracy is yet maintained even if incomplete grasps occur frequently.

\subsection{Object scalability}

To test the ability to classify objects of different scale, we train the classifiers as described in Section \ref{sec:classification_naive} while normalizing the grasp parameterization vectors. As described in Section \ref{sec:scale}, two similar grasps with a scaling factor reference will have the same parameterization vector if the contact points of each grasp are scaled by the square root of the total surface area ($A_v$). Hence, we build a training set for the eight KIT objects comprising of labeled and normalized parameterization vectors of sampled grasps. Evaluation was then performed for each object over 1,000 trials. In each trial, the object was randomly scaled by $s\in[0.1, 5]$ followed by grasp sampling and prediction using IC and BC. Table \ref{tb:scaled_objects} summarizes the success rates for 3-, 4- and 5-finger grasps. The success rates are slightly lower compared to non-scaled objects. IC still performs poorly without normals and when using a low number of fingers. It is important to note that scaling down objects may hinder the ability to acquire valid normal vectors since the robot hand does not also scale along. Nevertheless, the results show that the classification can be scale invariant and generalized to objects of different size even without normal information.

\begin{table*}[h!]
\small
\parbox{.6\linewidth}{
\centering
\caption{Success rate and sampling improvement when using $z$-finger grasps}
\label{tb:3_in_objects}
\begin{tabular}{|l|p{0.45cm}p{0.45cm}|p{0.45cm}p{0.45cm}|p{0.45cm}p{0.45cm}|p{0.45cm}p{0.45cm}|}\hline
$n$         & \multicolumn{2}{c|}{4} & \multicolumn{2}{c|}{5} & \multicolumn{2}{c|}{6} & \multicolumn{2}{c|}{7} \\
Normals              & w/   & w/o       & w/        & w/o       & w/     & w/o  & w/     & w/o \\\hline
IC & \multicolumn{8}{c|}{IC, $z=3$} \\\hline
\multicolumn{1}{|l|}{Success rate (\%)} & 95.7 & 62 & 96.7 & 62.8 & 96.1 & 64 & 96.8 & 62.3\\
\multicolumn{1}{|l|}{Num. samples} & 1.6 & 1.26 & 1.57 & 1.25 & 1.57 & 1.26 & 1.49 & 1.28\\
\multicolumn{1}{|l|}{Samp. redu. (\%)} & 63.2 & 65 & 66.3 & 71.4 & 60.4 & 69.9 & 62.6 & 71.4\\
\hline
IC & \multicolumn{8}{c|}{IC, $z=4$}  \\\hline
\multicolumn{1}{|l|}{Success rate (\%)} & - & - & 93.5 & 72.0 & 95.0 & 76.2 & 97.0 & 76.7 \\
\multicolumn{1}{|l|}{Num. samples} & - & - & 1.75 & 1.58 & 2.2 & 1.26 & 2.2 & 1.43 \\
\multicolumn{1}{|l|}{Samp. redu. (\%)} & - & - & 62.5 & 63.8 & 44.1 & 69.9 & 45.3 & 68.1  \\
\hline
BC-NP & \multicolumn{8}{c|}{BC-NP. $z=3$}  \\\hline
\multicolumn{1}{|l|}{Success rate (\%)} & 92.5 & 91.9 & 91.7 & 93.1 & 93.4 & 96.2 & 94.5 & 96.1  \\
\multicolumn{1}{|l|}{Num. samples} & 1.33 & 4.2 & 1.27 & 4.65 & 1.32 & 4.8 & 1.3 & 5.2  \\
\multicolumn{1}{|l|}{Samp. redu. (\%)} & 60.8 & 35.7 & 60.3 & 8.3 & 49.8 & -35 & 40.2 & -86  \\
\hline
BC-NP & \multicolumn{8}{c|}{BC-NP. $z=4$}  \\\hline
\multicolumn{1}{|l|}{Success rate (\%)} & - & - & 96.0 & 89.0 & 95.5 & 91.0 & 97.2 & 94.0 \\
\multicolumn{1}{|l|}{Num. samples}      & - & - & 1.36 & 2.48 & 1.29 & 2.53 & 1.31 & 2.15 \\
\multicolumn{1}{|l|}{Samp. redu. (\%)}  & - & - & 60 & 51.1 & 59.6 & 28.9 & 50.2 & 22.6 \\
\hline
\end{tabular}
}
\hfill
\parbox{.4\linewidth}{
\centering
\caption{Success rate and number of samples for incomplete grasps}
\label{tb:p(z)}
\begin{tabular}{|l|p{0.45cm}p{0.45cm}|p{0.45cm}p{0.45cm}|}\hline
$n$, $p_n(z)$         & \multicolumn{2}{c|}{4, $p_4(z)$} & \multicolumn{2}{c|}{5, $p_5(z)$} \\
Normals              & w/   & w/o       & w/        & w/o  \\\hline
 & \multicolumn{4}{c|}{IC} \\\hline
\multicolumn{1}{|l|}{Success rate (\%)} & 98.7 & 75.8 & 98.8 & 87.1 \\
\multicolumn{1}{|l|}{Num. samples} & 4.99 & 3.47 & 4.19 & 3.56 \\
\hline
 & \multicolumn{4}{c|}{BC-NP} \\\hline
\multicolumn{1}{|l|}{Success rate (\%)} & 97.3 & 95.3 & 93.3 & 96.0 \\
\multicolumn{1}{|l|}{Num. samples} & 3.27 & 7.72 & 2.87 & 6.3 \\
\hline
\end{tabular}
}
\end{table*}

\begin{figure}
\centering
\includegraphics[width=\linewidth]{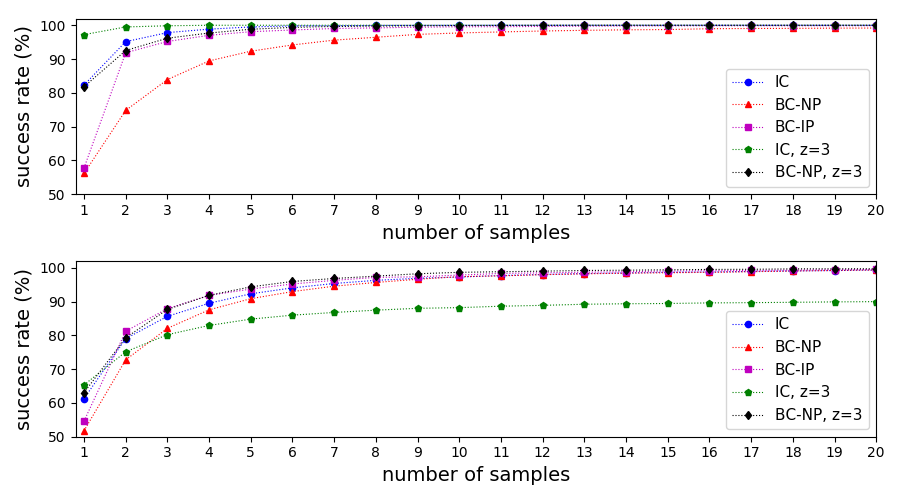}
\caption{Classification success rate of the eight objects for 4-finger grasps (top) with and (bottom) without using the normals at the contacts.}
\label{fig:iterations_objs4}
\end{figure}

\subsection{Geometry recognition}
\label{sec:geometry_recognition}

In this section, we evaluate the ability of the approach to classify everyday objects into types of geometries and not only specific objects. First, we train a classifier using three selected primitive geometries: a box, a sphere and a cylinder. To generate a training set for the classifier, three spatial CAD models were formed in an arbitrary size and positioned aligned to the primary axes as seen in Figure \ref{fig:bcs}. Their general size is not of an importance as further discussed. Then, for each object, 1,000 variations were created by non-uniform resizing along the primary axes with a random factor of $s_k\in[0.5,1.5]$, where $k=x,y,z$. Thus, we acquire three sets associated with the primitive geometries. Due to the non-uniform resizing of the sphere and cylinder, their corresponding sets include ellipsoids and elliptic cylinders of various sizes, respectively. Since the primitive objects and future query ones have arbitrary sizes, grasp samples are required to be normalized in order to compare similar geometries with different sizes. Therefore, each sampled grasp is scaled by the surface area of its grasp polyhedron as described in Section \ref{sec:scale}. According to Theorem \ref{thm:equal_phi}, such scaling will enable invariant grasp representation and allow to match between grasps of objects with different sizes. Training data was acquired by sampling, normalizing, parameterizing and labeling random grasps within the three sets. 

\begin{figure}
\centering
\begin{tabular}{ccc}
\includegraphics[width=0.2\linewidth]{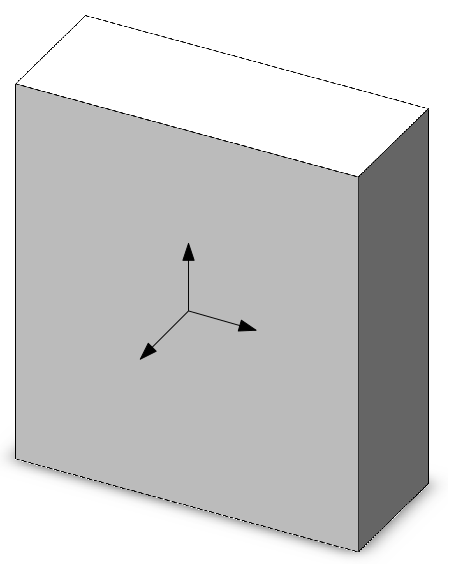} \text{~~~~} & \includegraphics[width=0.1\linewidth]{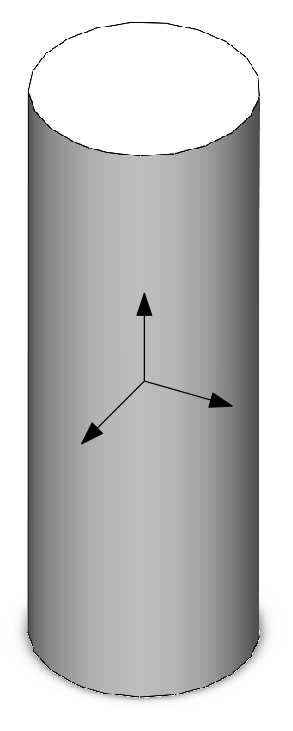} & \text{~~~~}
\includegraphics[width=0.2\linewidth]{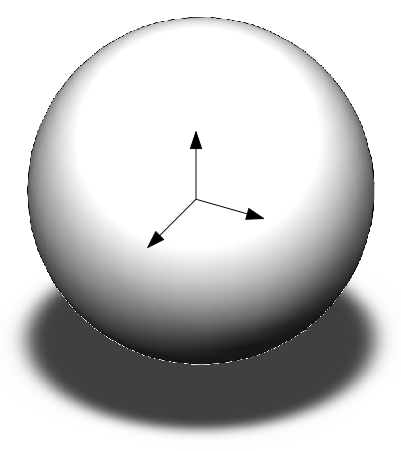} \\
(a) \text{~~~~} & (b) & \text{~~~~} (c) \\
\end{tabular}
\caption{Three primitive geometries used to train a classifier to classify types of objects: (a) box, (b) cylinder and (c) sphere.}
\label{fig:bcs}
\end{figure}

\begin{table}[]
\centering
\caption{Success rate (\%) for classifying scaled objects}
\label{tb:scaled_objects}
\begin{tabular}{|l|cc|cc|cc|}\hline
$n$        & \multicolumn{2}{c|}{3} & \multicolumn{2}{c|}{4} & \multicolumn{2}{c|}{5} \\
Normals              & w/   & w/o       & w/        & w/o       & w/     & w/o    \\\hline
\multicolumn{1}{|l|}{IC}    & 96.2 & 43.3 & 90.2 & 63.8 & 91.1 & 73.5 \\
\multicolumn{1}{|l|}{BC-NP} & 91.5 & 84.0 & 92.2 & 91.3 & 85.2 & 90.9 \\
\multicolumn{1}{|l|}{BC-IP} & 93.25 & 86.8 & 91.0 & 90.5 & 90.2 & 91.8 \\
\hline
\end{tabular}
\end{table}

\begin{figure}
\centering
\includegraphics[width=\linewidth]{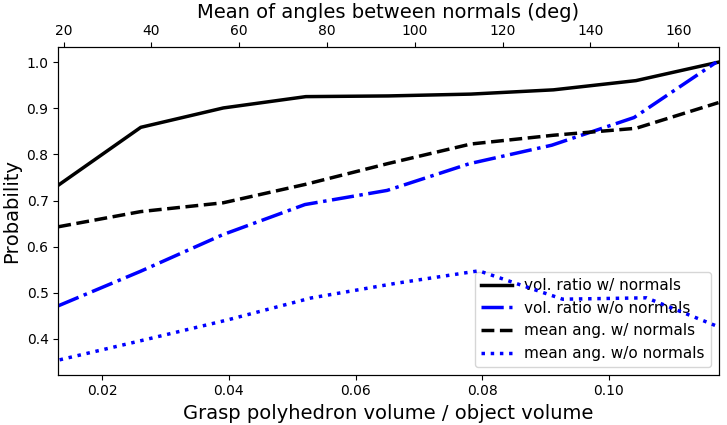}
\caption{The average probability of a NN classification with regards to the volume ratio between the grasp polyhedron and the object, and the mean angles of between the normals at the contacts. The mean of the angles is computed using the mean of circular quantities.}
\label{fig:grasp_quality}
\end{figure}

We now explore the performance of IC and BC-NP, trained with the primitive data, in classifying the geometry of 33 everyday objects taken from the Graspnet dataset \citep{fang2020}. Similar to the experiments in Section \ref{sec:classification_naive} and according to Algorithms \ref{alg:iterative}-\ref{alg:bayesian_classification}, a set of 4-finger grasp samples are taken from each object until the lower bound certainty ($\lambda_s=\lambda_p=0.85$) is achieved. We note that grasp samples taken from each query object are also normalized according to Theorem \ref{thm:equal_phi}. Table \ref{tb:graspnet} presents, for each object, classification rate into each of the three primitives out of 10 trials. The results do not consider the normals at the contacts to demonstrate classification with minimal information. Experiments with normals have not shown significant advantage in this case. The results show that most of the objects are classified as expected with high success rate. IC is slightly more challenged in identifying boxes than BC-NP. Some objects, such as \textit{dabao\_wash\_soup}, \textit{soap} and \textit{dove}, have flat faces along with large rounded corners that can be classified differently based on the samples taken. Hence, their classification rates are distributed between the different primitives. Nevertheless, our proposed approach have been successfully demonstrated to  distinguish between different geometries of everyday objects.

\begin{table*}[]
\centering
\small
\caption{Results for classifying objects from the Graspnet \citep{fang2020} dataset into three types of geometries using 4-finger grasps and without considering normals at the contacts}
\label{tb:graspnet}
\begin{tabular}{|ccc|ccc|ccc|}
\hline
\multirow{2}{*}{id}  & \multirow{2}{*}{Image} & \multirow{2}{*}{Name} & \multicolumn{3}{c|}{IC (rate in \%)} & \multicolumn{3}{c|}{BC-NP (rate in \%)}  \\ 
  &  &  & ~~~~~box~~~~~ & el. cylinder & ellipsoid & ~~~~~box~~~~~ & el. cylinder & ellipsoid   \\ \hline

\rule{0pt}{15pt} 000 & \parbox[c]{25pt}{\includegraphics[height=4ex]{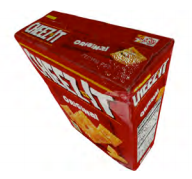}}   & 003\_cracker\_box & \cellcolor[HTML]{9B9B9B}50 & 0 & \cellcolor[HTML]{9B9B9B}50  & \cellcolor[HTML]{9B9B9B}90  & 10 & 0  \\ 

\rule{0pt}{16pt} 001 & \parbox[c]{20pt}{\includegraphics[height=4ex]{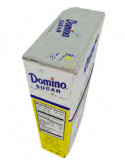}}   & 004\_sugar\_box & \cellcolor[HTML]{9B9B9B}60   & 10 & 30    & \cellcolor[HTML]{9B9B9B}80    & 20 & 0  \\ 


\rule{0pt}{15pt} 003 & \parbox[c]{20pt}{\includegraphics[height=4ex]{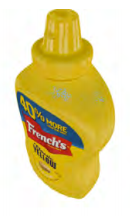}}   & 006\_mustard\_bottle & 30   & \cellcolor[HTML]{9B9B9B}50 & 20 & 40       & \cellcolor[HTML]{9B9B9B}60        & 0        \\ 


\rule{0pt}{15pt} 005 & \parbox[c]{20pt}{\includegraphics[height=4ex]{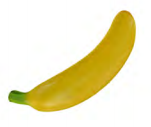}}   & 011\_banana & 0 & \cellcolor[HTML]{9B9B9B}100 & 0 & 0        & \cellcolor[HTML]{9B9B9B}100 & 0        \\ 


\rule{0pt}{15pt} 010 & \parbox[c]{10pt}{\includegraphics[height=4ex]{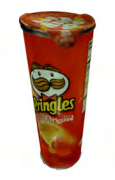}}   & 001\_chips\_can & 10    & \cellcolor[HTML]{9B9B9B}90 & 10 &   0       &  \cellcolor[HTML]{9B9B9B}100 & 0     \\ 

\rule{0pt}{15pt} 011 & \parbox[c]{20pt}{\includegraphics[height=3ex]{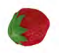}}   & 012\_strawberry & 0 & 0 & \cellcolor[HTML]{9B9B9B}100 & 0 & 0 & \cellcolor[HTML]{9B9B9B}100   \\ 

\rule{0pt}{15pt} 012 & \parbox[c]{20pt}{\includegraphics[height=4ex]{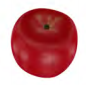}}   & 013\_apple & 10 & 10 & \cellcolor[HTML]{9B9B9B}80 & 20 & 0 & \cellcolor[HTML]{9B9B9B}80   \\ 

\rule{0pt}{15pt} 013 & \parbox[c]{20pt}{\includegraphics[height=4ex]{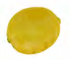}}   & 014\_lemon & 0 & 0 & \cellcolor[HTML]{9B9B9B}100 & 0 & 0 & \cellcolor[HTML]{9B9B9B}100   \\ 

\rule{0pt}{15pt} 014 & \parbox[c]{20pt}{\includegraphics[height=4ex]{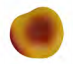}}   & 015\_peach & 10 & 0 & \cellcolor[HTML]{9B9B9B}90 & 0 & 0 & \cellcolor[HTML]{9B9B9B}100   \\ 

\rule{0pt}{15pt} 015 & \parbox[c]{10pt}{\includegraphics[height=4ex]{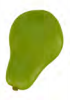}}   & 016\_pear & 10 & 0 & \cellcolor[HTML]{9B9B9B}90 & 0 & 10 & \cellcolor[HTML]{9B9B9B}90   \\ 

\rule{0pt}{15pt} 016 & \parbox[c]{20pt}{\includegraphics[height=4ex]{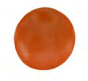}}   & 017\_orange & 10 & 0 & \cellcolor[HTML]{9B9B9B}90 & 0 & 0 & \cellcolor[HTML]{9B9B9B}100   \\ 

\rule{0pt}{15pt} 017 & \parbox[c]{20pt}{\includegraphics[height=4ex]{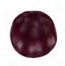}}   & 018\_plum & 0 & 0 & \cellcolor[HTML]{9B9B9B}100 & 0 & 0 & \cellcolor[HTML]{9B9B9B}100   \\ 

\rule{0pt}{15pt} 019 & \parbox[c]{10pt}{\includegraphics[height=4ex]{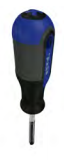}}   & 043\_phillips\_screwdriver & 0 & \cellcolor[HTML]{9B9B9B}100 & 0 & 0 & \cellcolor[HTML]{9B9B9B}100 & 0   \\ 

\rule{0pt}{15pt} 020 & \parbox[c]{10pt}{\includegraphics[height=4ex]{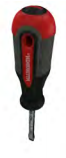}}   & 043\_flat\_screwdriver & 0 & \cellcolor[HTML]{9B9B9B}100 & 0 & 0 & \cellcolor[HTML]{9B9B9B}100 & 0   \\ 

\rule{0pt}{15pt} 021 & \parbox[c]{20pt}{\includegraphics[height=3ex]{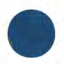}}   & 057\_racquetball & 0 & 0 & \cellcolor[HTML]{9B9B9B}100 & 0 & 0 & \cellcolor[HTML]{9B9B9B}100   \\ 

\rule{0pt}{15pt} 024 & \parbox[c]{20pt}{\includegraphics[height=3ex]{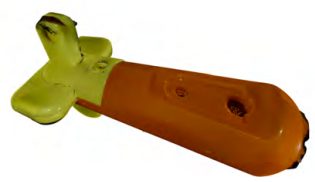}}   & 072\_a\_toy\_airplane & 20 & \cellcolor[HTML]{9B9B9B}70 & 10 & 0 & \cellcolor[HTML]{9B9B9B}100 & 0   \\ 

\rule{0pt}{15pt} 035 & \parbox[c]{20pt}{\includegraphics[height=4ex]{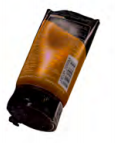}}   & jvr\_cleansing\_foam & 0 & \cellcolor[HTML]{9B9B9B}100 & 0 & 10 & \cellcolor[HTML]{9B9B9B}90 & 0   \\ 

\rule{0pt}{15pt} 036 & \parbox[c]{20pt}{\includegraphics[height=4ex]{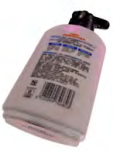}}   & dabao\_wash\_soup & \cellcolor[HTML]{9B9B9B}50 & 30 & 20 & 0 & \cellcolor[HTML]{9B9B9B}100 & 0   \\ 

\rule{0pt}{15pt} 037 & \parbox[c]{20pt}{\includegraphics[height=4ex]{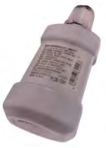}}   & nzskincare\_mouth\_rinse & 10 & \cellcolor[HTML]{9B9B9B}70 & 20 & 0 & \cellcolor[HTML]{9B9B9B}100 & 0   \\ 

\rule{0pt}{15pt} 038 & \parbox[c]{20pt}{\includegraphics[height=4ex]{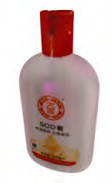}}   & dabao\_sod & 10 & \cellcolor[HTML]{9B9B9B}80 & 10 & \cellcolor[HTML]{9B9B9B}60 & 40 & 0   \\ 

\rule{0pt}{15pt} 039 & \parbox[c]{20pt}{\includegraphics[height=4ex]{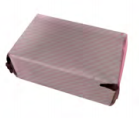}}   & soap\_box & 20 & 10 & \cellcolor[HTML]{9B9B9B}70 & \cellcolor[HTML]{9B9B9B}80 & 0 & 20   \\ 

\rule{0pt}{15pt} 040 & \parbox[c]{20pt}{\includegraphics[height=4ex]{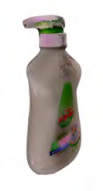}}   & kispa\_cleanser & 0 & \cellcolor[HTML]{9B9B9B}70 & 30 & 0 & \cellcolor[HTML]{9B9B9B}100 & 0   \\ 

\rule{0pt}{15pt} 041 & \parbox[c]{20pt}{\includegraphics[height=4ex]{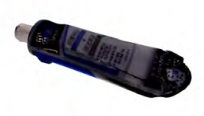}}   & darlie\_toothpaste & 0 & \cellcolor[HTML]{9B9B9B}90 & 10 & 0 & \cellcolor[HTML]{9B9B9B}100 & 0   \\ 


\rule{0pt}{15pt} 043 & \parbox[c]{20pt}{\includegraphics[height=4ex]{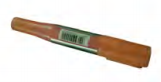}}   & baoke\_marker & 0 & \cellcolor[HTML]{9B9B9B}100 & 0 & 0 & \cellcolor[HTML]{9B9B9B}100 & 0   \\ 

\rule{0pt}{15pt} 044 & \parbox[c]{20pt}{\includegraphics[height=4ex]{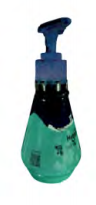}}   & hosjam & 10 & \cellcolor[HTML]{9B9B9B}80 & 10 & 0 & \cellcolor[HTML]{9B9B9B}100 & 0   \\ 

\rule{0pt}{15pt} 058 & \parbox[c]{20pt}{\includegraphics[height=4ex]{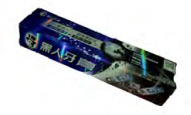}}   & darlie\_box & 0 & \cellcolor[HTML]{9B9B9B}100 & 0 & 0 & \cellcolor[HTML]{9B9B9B}100 & 0   \\ 

\rule{0pt}{15pt} 059 & \parbox[c]{20pt}{\includegraphics[height=4ex]{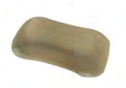}}   & soap & 10 & 40 & \cellcolor[HTML]{9B9B9B}50 & \cellcolor[HTML]{9B9B9B}50 & 20 & 30   \\ 

\rule{0pt}{15pt} 061 & \parbox[c]{10pt}{\includegraphics[height=4ex]{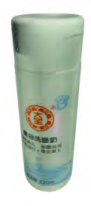}}   & dabao\_facewash & 0 & \cellcolor[HTML]{9B9B9B}80 & 20 & 0 & \cellcolor[HTML]{9B9B9B}100 & 0   \\ 

\rule{0pt}{15pt} 062 & \parbox[c]{10pt}{\includegraphics[height=4ex]{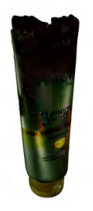}}   & pantene & 20 & \cellcolor[HTML]{9B9B9B}70 & 10 & 0 & \cellcolor[HTML]{9B9B9B}100 & 0   \\ 

\rule{0pt}{15pt} 063 & \parbox[c]{10pt}{\includegraphics[height=4ex]{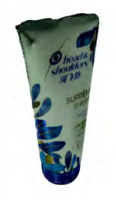}}   & head\_shoulders\_supreme & 20 & \cellcolor[HTML]{9B9B9B}80 & 0 & 0 & \cellcolor[HTML]{9B9B9B}100 & 0   \\ 

\rule{0pt}{15pt} 064 & \parbox[c]{10pt}{\includegraphics[height=4ex]{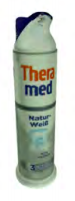}}   & thera\_med & 10 & \cellcolor[HTML]{9B9B9B}80 & 10 & 0 & \cellcolor[HTML]{9B9B9B}100 & 0   \\ 

\rule{0pt}{15pt} 065 & \parbox[c]{10pt}{\includegraphics[height=4ex]{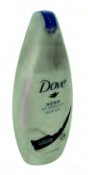}}   & dove & 20 & \cellcolor[HTML]{9B9B9B}40 & \cellcolor[HTML]{9B9B9B}40 & 0 & \cellcolor[HTML]{9B9B9B}100 & 0   \\ 

\rule{0pt}{15pt} 066 & \parbox[c]{10pt}{\includegraphics[height=4ex]{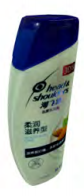}}   & head\_shoulders\_care & 0 & \cellcolor[HTML]{9B9B9B}90 & 10 & 0 & \cellcolor[HTML]{9B9B9B}100 & 0   \\ 


\hline
\end{tabular}
\end{table*}

\subsection{Grasp quality}
We characterize the quality of a grasp in terms of its ability to provide information about a query object. Hence, we analyse the classification certainty that a grasp provides with regards to the properties of its parameterization. We consider two measures. First, we measure the ratio between the volume of the polyhedron formed by the contact points and the volume of the object. Second, we measure the mean angles between all normals at the contact points computed by the mean of circular quantities. The volume and mean angles measure the distribution of the contacts and the variation of normal directions, respectively. 

Figure \ref{fig:grasp_quality} presents results for NN classification certainty with regards to these measures for 4-finger grasps with and without normals. The certainty is expressed in the probability distribution assigned to the query object by the classifier. It is clear that increasing the grasp polyhedron's volume yields better prediction certainty. A higher volume means better sampling of shape and size and a more informative classifier input. Similarly, when considering normals, a high variation in the normal directions implies about the object shape. When not considering normals in the classifier, they do not provide additional information as expected and, therefore, the classification probability is low. These insights are similar to human behaviour where tactile recognition of an object is usually performed with a caging formation of the hand. On the other hand, positioning all of the fingers on one facet of a box, for example, provides deficient information.



\section{Gripper Experiments}
\label{sec:experiments}

\begin{figure}
\centering
\includegraphics[width=\linewidth]{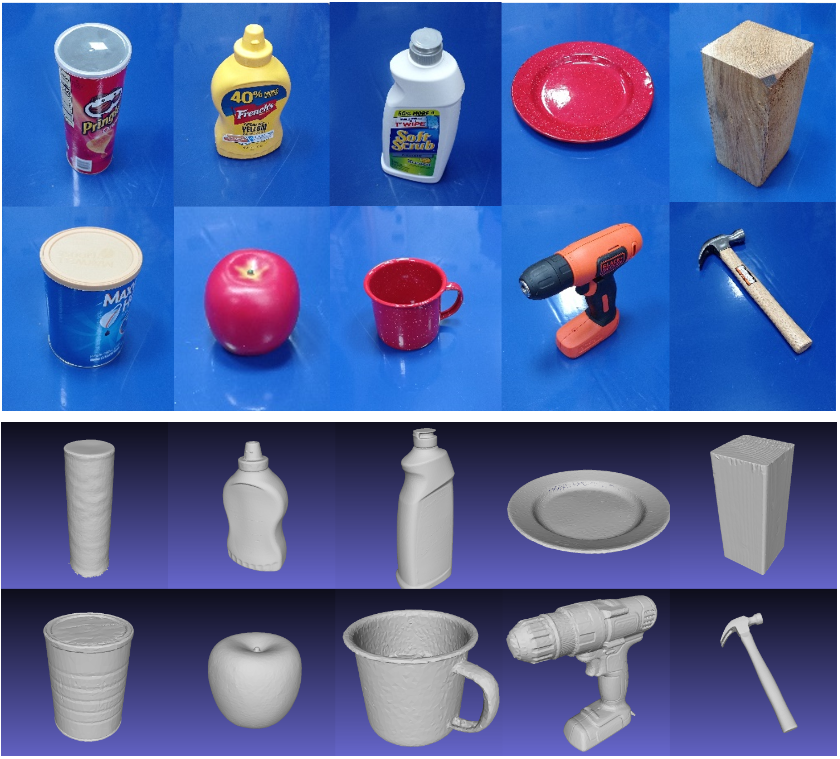}\\
\caption{Ten YCB objects used in the grasping experiments (top) and their meshes (bottom) used for training the classifiers. From top left: chips can (01), mustard bottle (06), bleach cleanser (21), plate (29), wood block (36), master chef can (02), apple (13), mug (25), power drill (35) and hammer (48) \citep{Calli2017}.}
\label{fig:ycb}
\end{figure}

In this section, we conduct experiments while considering real grasps with robotic hands. In all experiments, we evaluate recognition capabilities with the Yale-CMU-Berkeley (YCB)  dataset \citep{Calli2017}. The YCB dataset includes models of various objects specifically designed for benchmarking in grasping and manipulation research. We use 10 objects seen in Figure \ref{fig:ycb}: chips can (YCB object number 01), master chef can (02), mustard bottle (06), apple (13), bleach cleanser (21), mug (25), plate (29), power drill (35), wood block (36) and hammer (48).

We experiment with a simulated 5-finger Pisa/IIT SoftHand and a real 4-finger Allegro hand. The Pisa/IIT SoftHand \citep{Catalano2014} was simulated in ROS-Gazebo environment. The Pisa/IIT hand is an anthropomorphic hand with 19 joints and adaptive synergy capabilities. Hence, it uses only one actuator and a system of ligaments and tendons to enable adaptive grasping. In the simulation, seen in Figure \ref{fig:pisa}, we assume that finger joint angles can be measured and the kinematics are known. That is, the locations of the finger tips can be computed at any time. Joint torques and, therefore, normal directions cannot be measured. In real experiments, we used the 4-finger Allegro robotic hand. The Allegro is a fully-actuated hand comprised of 16 actuators, four in each finger. The hand is controlled through ROS which also provides data stream of joint angles and torques. Here also, the locations of the finger tips are computed using the kinematics \eqref{eq:hand_kinematics} with extracted joint angles. In addition, we compute the contacts force directions using angles and torques according to \eqref{eq:contact_forces}. Grasps of various objects with the Allegro hand during experiments can be seen in Figures \ref{fig:allegro_grasp} and \ref{fig:allegro_ycb}. Figures \ref{fig:mustard_allegro}-\ref{fig:drill_allegro} (see also Extension 1) show snapshots of grasp iterations for four YCB objects and the certainty about the object increasing when using BC-NP ($z=3$).

Experiments with both hands are performed after training classifiers using only the CAD models of the 10 objects as described in Section \ref{sec:obj_classification}. Here also, the classifiers are independent of the robotic hands to be used. In both the simulated and real hands, we acquire the position of the finger tips through forward kinematics. However, it is challenging to exactly measure the locations of the contact points. This tends to confuse a classifier trained over exact positions on the object. Hence, we add Gaussian noise to the contact locations of the CAD models during the generation of training data. Consequently, the classifier becomes more robust and immune to inaccuracies in contact locations. 

Classification results for 10 queries per object can be seen in Table \ref{tb:experiments}. In each each query, grasp samples are taken randomly by arbitrary choosing approach angles of the hand to grasp the object. When using a robotic hand, the distribution of contact points is limited based on its kinematics. Hence, the number of iterations required for sufficient certainty is higher. The success rates for BC-NP are the highest while requiring relatively many grasp samples. Nevertheless, when using $z$-finger grasps of $z=3$, the number of iterations is reduced with the cost of a slightly lower accuracy. Here also, IC with $z=3$ result in inaccurate classification. The use of normals at the contact points acquired from the Allegro hand did not improve accuracy much due to noisy torque measurements. Nevertheless, such addition of data enabled the reduction of the number of iterations. While five or six grasps can be seen as too much, we note that grasps do not have to be significantly different. In our experiments, sliding of a finger from a former grasp configuration to a near region was also considered as a new grasp sample. Examples for this can be seen in Figures \ref{fig:mustard_allegro}-\ref{fig:drill_allegro}. Thus, sampling can be fast and efficient, very similar to human tactile motion.

\begin{figure}
\centering
\includegraphics[width=\linewidth]{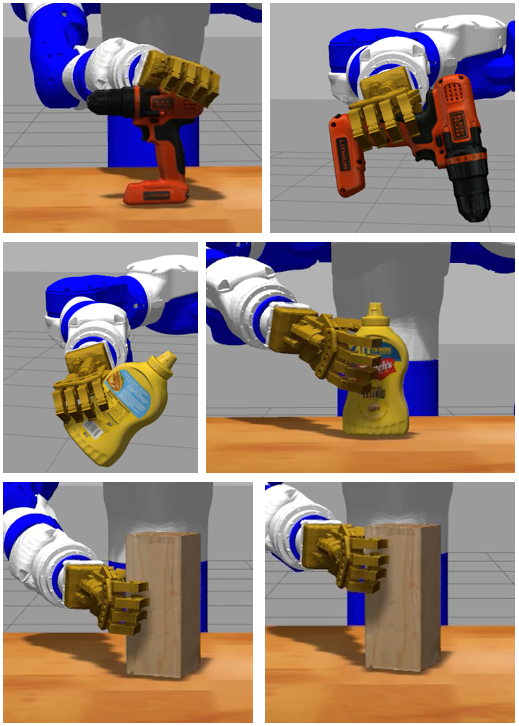}
\caption{Three YCB objects grasped by the 5-finger Pisa/IIT soft hand in various grasps: (top) power drill, (middle) Mustard bottle and (bottom) wood block.}
\label{fig:pisa}
\end{figure}
\begin{figure}
\centering
\includegraphics[width=\linewidth]{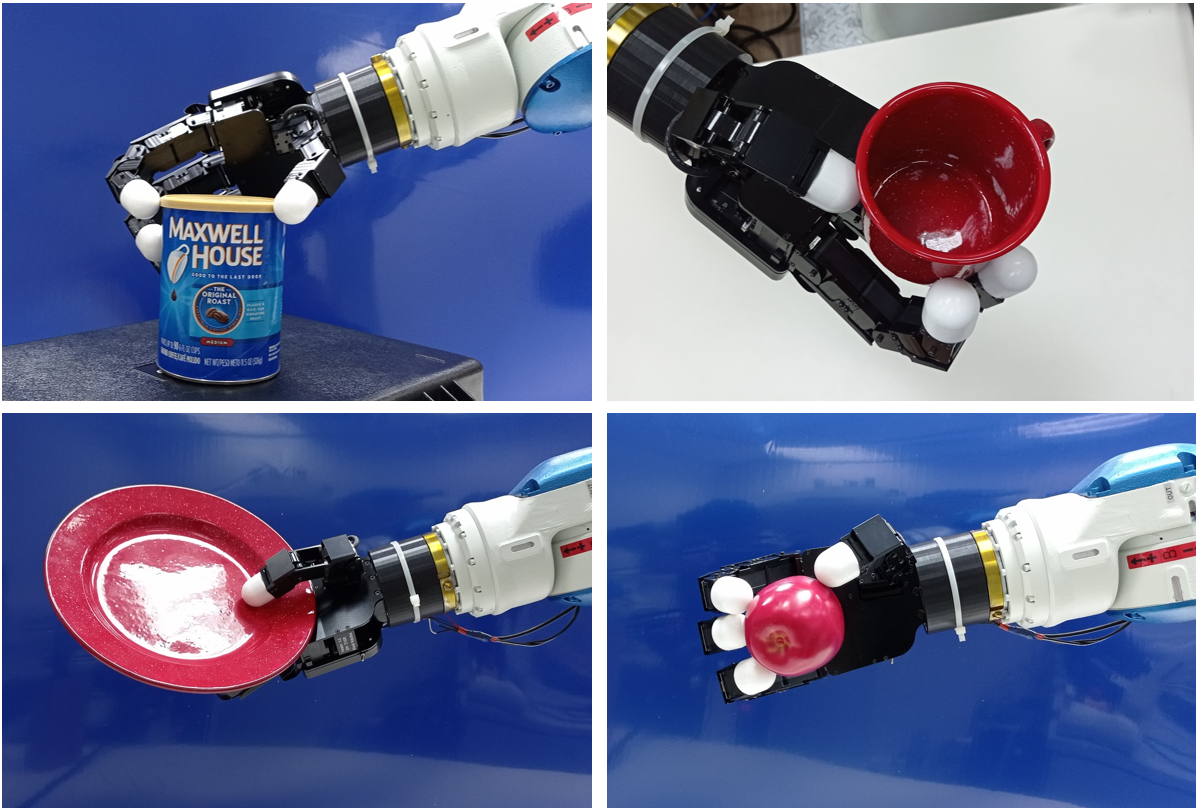}
\caption{Object recognition experiments with the 4-finger Allegro hand. The images show grasp instances of four YCB objects.}
\label{fig:allegro_ycb}
\end{figure}

\begin{figure*}
\centering
\includegraphics[width=\linewidth]{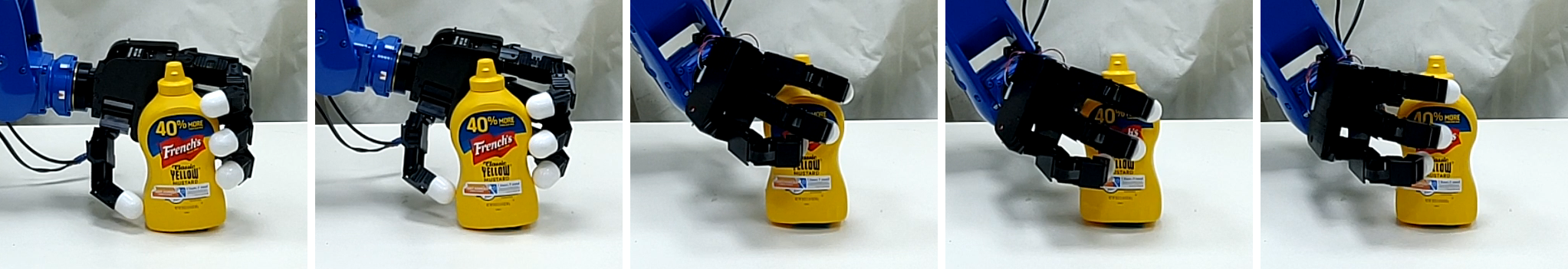}
\caption{Five iterative grasps of the Mustard Bottle object in an experiment with the 4-finger Allegro hand. The recognition certainties about the object using BC-NP ($z=3$) without normals are, from left to right, $P_1(\ssl{O}_{\text{mustard}}|\ve{q}_1)=0.36$, $P_2(\ssl{O}_{\text{mustard}}|\ve{q}_1, \ve{q}_2)=0.51$, $P_2(\ssl{O}_{\text{mustard}}|\ve{q}_1, \ve{q}_2, \ve{q}_3)=0.66$, $P_4(\ssl{O}_{\text{mustard}}|\ve{q}_1, \ve{q}_2, \ve{q}_3,\ve{q}_4)=0.84$ and $P_5(\ssl{O}_{\text{mustard}}|\ve{q}_1, \ve{q}_2, \ve{q}_3,\ve{q}_4,\ve{q}_5)=0.90$.}
\label{fig:mustard_allegro}
\end{figure*}
\begin{figure*}
\centering
\includegraphics[width=\linewidth]{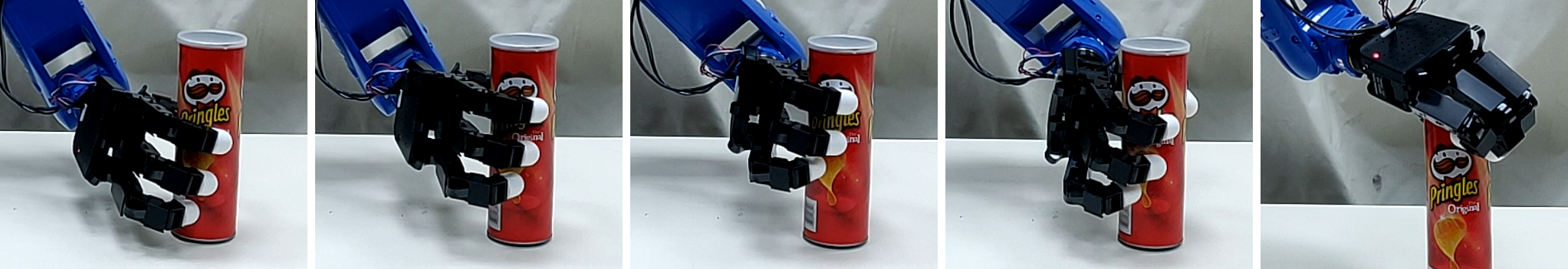}
\caption{Five iterative grasps of the Chips Can object in an experiment with the 4-finger Allegro hand. The recognition certainties about the object using BC-NP ($z=3$) without normals are, from left to right, $P_1(\ssl{O}_{\text{chips}}|\ve{q}_1)=0.23$, $P_2(\ssl{O}_{\text{chips}}|\ve{q}_1, \ve{q}_2)=0.39$, $P_2(\ssl{O}_{\text{chips}}|\ve{q}_1, \ve{q}_2, \ve{q}_3)=0.63$, $P_4(\ssl{O}_{\text{chips}}|\ve{q}_1, \ve{q}_2, \ve{q}_3,\ve{q}_4)=0.80$ and $P_5(\ssl{O}_{\text{chips}}|\ve{q}_1, \ve{q}_2, \ve{q}_3,\ve{q}_4,\ve{q}_5)=0.86$.}
\label{fig:chipscan_allegro}
\end{figure*}
\begin{figure*}
\centering
\includegraphics[width=\linewidth]{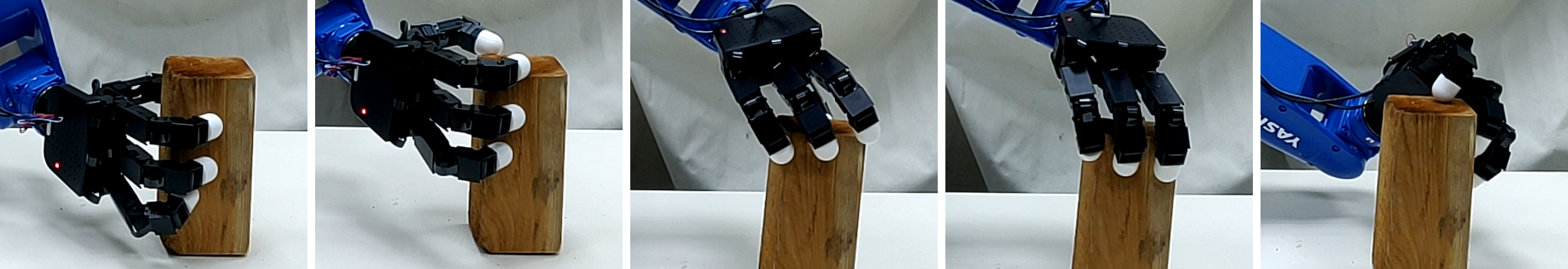}
\caption{Five iterative grasps of the Wood Block object in an experiment with the 4-finger Allegro hand. The recognition certainties about the object using BC-NP ($z=3$) without normals are, from left to right, $P_1(\ssl{O}_{\text{wood}}|\ve{q}_1)=0.20$, $P_2(\ssl{O}_{\text{wood}}|\ve{q}_1, \ve{q}_2)=0.18$, $P_2(\ssl{O}_{\text{wood}}|\ve{q}_1, \ve{q}_2, \ve{q}_3)=0.12$, $P_4(\ssl{O}_{\text{wood}}|\ve{q}_1, \ve{q}_2, \ve{q}_3,\ve{q}_4)=0.81$ and $P_5(\ssl{O}_{\text{wood}}|\ve{q}_1, \ve{q}_2, \ve{q}_3,\ve{q}_4,\ve{q}_5)=0.89$.}
\label{fig:woodblock_allegro}
\end{figure*}
\begin{figure*}
\centering
\includegraphics[width=\linewidth]{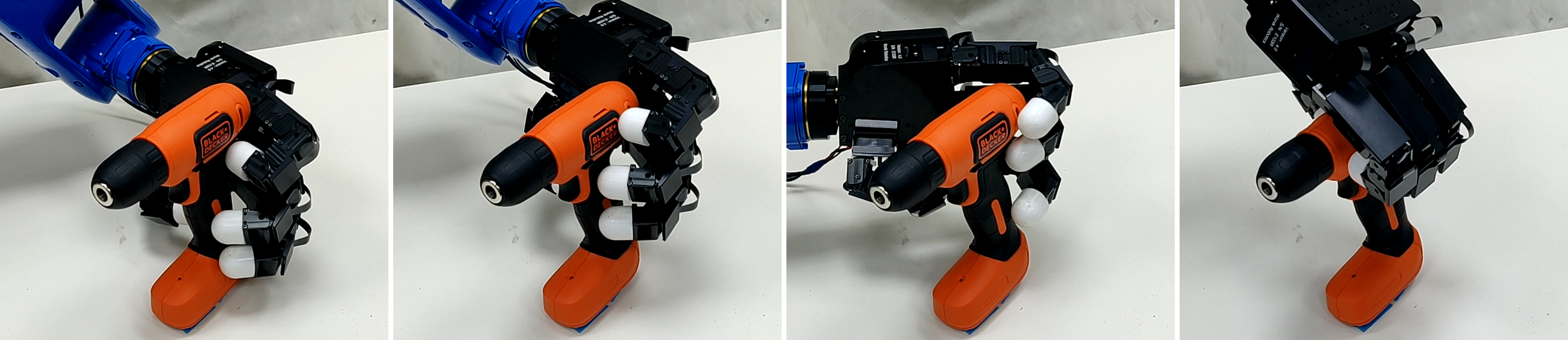}
\caption{Four iterative grasps of the Power Drill object in an experiment with the 4-finger Allegro hand. The recognition certainties about the object using BC-NP ($z=3$) without normals are, from left to right, $P_1(\ssl{O}_{\text{drill}}|\ve{q}_1)=0.29$, $P_2(\ssl{O}_{\text{drill}}|\ve{q}_1, \ve{q}_2)=0.33$, $P_2(\ssl{O}_{\text{drill}}|\ve{q}_1, \ve{q}_2, \ve{q}_3)=0.51$ and $P_4(\ssl{O}_{\text{drill}}|\ve{q}_1, \ve{q}_2, \ve{q}_3,\ve{q}_4)=0.87$.}
\label{fig:drill_allegro}
\end{figure*}



\begin{table*}[]
\centering
\small
\caption{Results for real hands experiments}
\label{tb:experiments}
\begin{tabular}{|l|l|c|cc|} \hline
\multicolumn{1}{|l}{}          &                        & Pisa hand & \multicolumn{2}{|c|}{Allegro hand} \\\hline
\multicolumn{2}{|c|}{n}                                 & 5         & \multicolumn{2}{c|}{4} \\\hline
\multicolumn{2}{|c|}{with/without normals}              & w/o       & w/o     & w/     \\\hline\hline
\multirow{2}{*}{IC}           & success rate           & 89\%      & 81\%    & 87\%   \\
                              & avg. num. of iteration & 3.76      & 2.97    & 3.78   \\\hline
\multirow{2}{*}{BC-NP}        & success rate           & 94\%      & 94\%    & 92\%   \\
                              & avg. num. of iteration & 8.15      & 10.45   & 4.80   \\\hline
\multirow{2}{*}{BC-IP}        & success rate           & 90\%      & 92\%    & 93\%   \\
                              & avg. num. of iteration & 4.86      & 9.18    & 4.17   \\\hline
\multirow{2}{*}{IC, $z=3$}    & success rate           & 51\%      & 52\%    & 49\%   \\
                              & avg. num. of iteration & 1.04      & 1.25    & 1.14   \\\hline
\multirow{2}{*}{BC-NP, $z=3$} & success rate           & 92\%      & 93\%    & 94\%   \\
                              & avg. num. of iteration & 4.22      & 4.85    & 1.62   \\\hline
\end{tabular}
\end{table*}

\section{Conclusions}

In this paper, we have proposed a novel approach for object classification through haptic glances without the use of tactile sensors. The approach is based on a unique representation of a grasp of an object. The grasp is represented as polyhedron formed by the contact points and described by a unique frame invariant parameterization. Grasp parameterization samples are solely taken from CAD models of a set of objects and used as training data for a classifier. A grasp of an object is an instance of the object's shape and may not embed sufficient information for a certain classification. Hence, we propose, observe and compare between two iterative methods where additional grasps are used to improve classification certainty. The approach was shown to be accurate both in a thorough analysis and in real hand experiments.

Future work may include smart haptic glances where individual fingers are moved in planned directions to improve classification certainty. The iterative algorithms can also integrate haptic glances from other sources such as visual perception and tactile sensors. In underactuated hands where the kinematics are not known, tactile sensors can be used along the fingers to identify where contact occurs and acquire multiple combinations of haptic glances for predictions. Hence, one grasp can provide several sequences of grasps and may remove the need to sample another grasp. In addition, we may wish to advance the method into non-rigid objects where finger pose variations are large.

\nocite{*} 
\bibliographystyle{SageH}
\bibliography{main_arxiv} 

\section*{Appendix A: Index to Multimedia Extensions}

\begin{table}[h!]
\centering
\caption*{Table of Multimedia extensions}
\begin{tabular}{llp{4cm}}
\hline
Extension & Media type & Description      \\ \hline
1         & Video      & Experiments where grasp iterations with the 4-finger Allegro hand are used to recognize several YCB objects. \\ \hline
\end{tabular}
\end{table}

\section*{Appendix B: Proof of Theorems}

Proof of Theorem \ref{thm:frame_invariance} is as follows.
\begin{proof}
Rotation $R$ and translation $\ve{d}$ represent a rigid transformation in the Euclidean space and therefore, preserves distance between points and angles between vectors \citep{Murray1994}. Hence, all lengths and angles of $\Theta(\ssl{P}_k)$ are equal to $\Theta(R\ssl{P}_k+\ve{d})$. Similarly, the angles between the vectors in $\ssl{N}_k$ to edges of $\Theta(\ssl{P}_k)$ are equal to the angles between the vectors in $R\ssl{N}_k$ to edges of $\Theta(R\ssl{P}_k+\ve{d})$. Consequently, both grasps are defined be the same polyhedron and will be parameterized equally by map $\Phi_n$.
\end{proof}

Next, we present the proof for Theorem \ref{thm:equal_phi}.
\begin{proof}
Scaling of the polyhedron $\Theta(\ssl{P}_i)$ by $\xi$ yields a new polyhedron $\Theta(\ssl{P}_j)$, where $\ssl{P}_i\in\ssl{G}_i$ and $\ssl{P}_j=\xi\cdot\ssl{P}_i\in\ssl{G}_j$, such that $\Theta(\ssl{P}_i)$ and $\Theta(\ssl{P}_j)$ are similar \citep{Osada2001}. Hence, all matching angles between both polyhedrons are equal. Furthermore, this imposes proportionality between any edge $e_{ik}\in\Theta(\ssl{P}_i)$ and its corresponding edge $e_{jk}\in\Theta(\ssl{P}_j)$, i.e., $\frac{e_{jk}}{e_{ik}}=\xi$. Also, if two polyhedrons are similar with a scale factor of $\xi$, then their surface areas are in the ratio of $\xi^2$. Therefore, for any given edge $k$, it follows that
\begin{equation}
    \xi^2=\left(\frac{e_{jk}}{e_{ik}}\right)^2=\frac{\sum_{u=1}^Ua_{ju}}{\sum_{u=1}^Ua_{iu}}
\end{equation}
where $a_{vu}$ is the surface area of facet $u$ ($u=1,...,U$) on polyhedron $\Theta(\ssl{P}_v)$. Again, let $A_v=\sqrt{\sum_{u=1}^Ua_{vu}}$, then it must be that
\begin{equation}
    \frac{e_{jk}}{e_{ik}}={\frac{A_j}{A_i}} \text{~~or~~} \frac{e_{jk}}{{A_j}}=\frac{e_{ik}}{{A_i}},
\end{equation}
for any two corresponding edges $e_{jk}$ and $e_{ik}$. Hence, scaling according to $\frac{1}{A_j}\ssl{P}_j$ and $\frac{1}{A_i}\ssl{P}_i$ will yield two equal polyhedrons. Consequently, the parameterization vectors $\Phi_n\left(\ssl{G}_i^{\left(A_i^{-1}\right)}\right)$ and $\Phi_n\left(\ssl{G}_j^{\left(A_j^{-1}\right)}\right)$ are equal.
\end{proof}

\end{document}